\documentclass{jmlr}

\usepackage[utf8]{inputenc} 
\usepackage[T1]{fontenc}    
\usepackage{hyperref}       
\usepackage{url}            
\usepackage{booktabs}       
\usepackage{amsfonts}       
\usepackage{nicefrac}       
\usepackage{microtype}      

\usepackage{amsmath,amsthm,amssymb,mathrsfs,amsfonts,dsfont,tikz}
\usepackage{algorithm,algorithmic}
\usepackage{multirow}
\usepackage{multicol,enumitem}

\usepackage{diagbox}
\usepackage{amsmath,amssymb}
\usepackage{amsmath,amsfonts,bm}
\usepackage{mathtools}

\usepackage{enumitem}
\usepackage{textcase,booktabs}
\usepackage{etoolbox}
\usepackage{tikz}
\usepackage{graphicx}

\usepackage{hyperref}
\hypersetup{final}

\title[Nearly Optimal Robust Method for Convex Compositional Problems]{Nearly Optimal Robust Method for Convex Compositional Problems with Heavy-Tailed Noise}

\author{\Name{Yan Yan}$^1$ \Email{yan-yan-2@uiowa.edu}\\
\Name{Xin Man}$^1$ \Email{xin-man@uiowa.edu}\\
\Name{Tianbao Yang}$^1$ \Email{tianbao-yang@uiowa.edu} \\
\addr $^1$Department of Computer Science, University of Iowa 
}

\newtheorem{rem}{Remark}
\newtheorem{assumption}{Assumption}

\newtheorem{thm}{Theorem}

\newtheorem{lem}{Lemma}

\def\E{\mathbb{E}}
\def\V{\mathbb{V}}
\def\P{\mathbb{P}}
\def\w{{\bf{w}}}
\def\z{{\bf{z}}}
\def\y{{\bf{y}}}
\def\calS{\mathcal S}
\def\calO{\mathcal O}

\def \R {\mathbb{R}}

\def \w {\mathbf{w}}

\def \x {\mathbf{x}}

\def \E {\mathrm{E}}

\def \x {\mathbf{x}}
\def \p {\mathbf{p}}

\def \b {\mathbf{b}}

\def \c {\mathbf{c}}

\def \z {\mathbf{z}}

\def \y {\mathbf{y}}
\def \u {\mathbf{u}}

\def \P {\mathcal{P}}

\def \wh {\widehat{\w}}

\def \C {\mathbf C}

\def \yh {\widehat\y}

\def \y {\mathbf{y}}
\def \E {\mathrm{E}}
\def \x {\mathbf{x}}

\def \z {\mathbf{z}}
\def \u {\mathbf{u}}

\def \w {\mathbf{w}}

\def \R {\mathbb{R}}

\def \c {\mathbf{c}}

\def \p {\mathbf{p}}

\def \b {\mathbf{b}}

\def \wh {\widehat{\w}}

\def \zh {\widehat{\z}}

\def \C {\mathcal{C}}

\def \P {\mathbb{P}}

\def \bmu {\boldsymbol{\mu}}

\begin{document}

\maketitle

\begin{abstract}
In this paper, we propose robust stochastic algorithms for solving convex compositional problems of the form $f(\E_\xi g(\cdot; \xi)) + r(\cdot)$ by establishing {\bf sub-Gaussian confidence bounds} under weak assumptions about the tails of noise distribution, i.e., {\bf heavy-tailed noise} with bounded second-order moments. One can achieve this goal by using an existing boosting strategy that  boosts a low probability convergence result into a high probability result. However, piecing together existing results for solving compositional problems suffers from several drawbacks: (i) the boosting technique requires strong convexity of the objective; (ii) it requires a separate algorithm to handle non-smooth $r$; (iii) it also suffers from an additional polylogarithmic factor of the condition number.   To address these issues, we directly develop a single-trial stochastic algorithm for minimizing optimal strongly convex compositional objectives, which has  a nearly optimal high probability convergence result matching the lower bound of stochastic strongly convex optimization up to a logarithmic factor.  To the best of our knowledge, this is the first work that establishes nearly optimal  sub-Gaussian confidence bounds for compositional problems under heavy-tailed assumptions. 
\end{abstract}

\section{Introduction}
In this paper, we consider the following stochastic compositional  problem:
\begin{align}\label{eqn:op}
\min_{\w\in\R^d} F(\w):=f(\E_{\xi}[g(\w; \xi)]) + r(\w), 
\end{align} 
where $\xi$ is a random variable, $f: \R^p\rightarrow \R$ is a smooth and Lipschitz continuous function and $r(\cdot)$ is convex and lower-semicontinuous.    
Denote by $g(\w) = \E_{\xi}[g(\w; \xi)]$. We focus on a family of problems where  $f(g(\w))$ is smooth convex but $F(\w)$ is optimal  strongly convex~\cite{doi:10.1137/140961134}. We aim to develop efficient stochastic algorithms enjoying  a sub-Gaussian confidence bound $\Pr(F(\bar\w) - \min_{\w}F(\w)\leq \epsilon)\geq 1-\delta$, under weak assumptions about the random function $g(\w; \xi)$ and its Jacobian. 

The above problem has a broad range of applications in machine learning, operation research, etc~\cite{DBLP:journals/mp/WangFL17}.  Although a  number of papers have proposed stochastic algorithms for solving the stochastic compositional problem~(\ref{eqn:op}), they are mostly low probability  convergence results. To the best of our knowledge, all of the existing results of~(\ref{eqn:op}) are expectational convergence results in the form of $\E[F(\bar \w) - \min_{\w}F(\w)]\leq \epsilon'$ with a polynomial sample complexity denoted by $\C(1/\epsilon')$~\footnote{Linear convergence has been established for finite-sum problems, which is not the focus  of this paper.}. By Markov inequality, this directly implies a low probability convergence result $\Pr(F(\bar \w) - \min_{\w}F(\w)\leq  \epsilon'/\delta)\geq 1 - \delta$ for a constant $\delta$. Hence, in order to guarantee $\Pr(F(\bar \w) - \min_{\w}F(\w)\leq  \epsilon)\geq 1 - \delta$ for a small $\epsilon$,  the sample complexity will be amplified by an undesired  polynomial factor of $1/\delta$, i.e., $\C(1/(\delta\epsilon))$. In this paper, we are interested in achieving a high probability convergence result with a logarithmic dependence on $1/\delta$, which is also known as {\bf sub-Gaussian confidence bound}.

High probability convergence results have been established for many stochastic convex optimization problems in the literature~\cite{Nemirovski:2009:RSA:1654243.1654247,DBLP:journals/jmlr/HazanK11a,RakhlinSS12,DBLP:journals/corr/abs-1607-01027,DBLP:journals/siamjo/GhadimiL13}. However, most of them are derived under bounded assumption or sub-Gaussian assumptions about stochastic functions. In practice, these assumptions could fail when data suffer from dramatic noise. For the robust learning problem considered later on (in the supplement), the individual loss $g(\w; \xi)$ and its gradient $\nabla g(\w; \xi)$ could be unbounded or follow a heavy-tailed distribution due to outliers or adversarial attack.   In this paper, we  avoid making such restrictive assumptions by imposing the following weak assumptions about the noisy function value and the noisy Jacobian of $g(\w)$: 
\begin{align}\label{eqn:ass}
\E[\| \nabla g(\w; \xi) \|^2] \leq C_g^2,\quad 
\E[\| \nabla g(\w; \xi) - \nabla g(\w) \|^2] \leq \sigma_1^2, \quad
\E[\| g(\w; \xi) - g(\w) \|^2] \leq \sigma_0^2,
\end{align}
where $\|\cdot\|$ denotes the Euclidean norm of a vector or the spectral norm of a matrix. 

Recently, Davis et al.~\cite{DBLP:journals/corr/abs/1907.13307} proposed a generic boosting technique that can boost a low probability convergence result into a high probability convergence result for minimizing a strongly convex objective function. It comes at a cost increase that depends on a polylogarithmic factor of the condition number. Combining this boosting technique with existing stochastic algorithms for~(\ref{eqn:op}) (e.g.,~\cite{DBLP:journals/mp/WangFL17}) can provide a solution to achieving a high probability convergence result. However, such a ``lazy" solution is not necessarily the best solution.  First,  the boosting technique requires strong convexity of the objective; Second,  it requires a separate algorithm to handle non-smooth $r$; Third,  it also suffers from an additional polylogarithmic factor of the condition number.  
We make the following contributions to address these issues. 
\begin{itemize}[leftmargin=*]
\item First, we propose a  simple robust algorithm  for  optimal  strongly convex function $F(\w)$ by using robust mean estimation technique. It enjoys a sample complexity in the order of $O(\frac{\log(\kappa/\delta)}{\mu^2\epsilon})$ for achieving $\Pr(F(\bar\w) - \min_{\w}F(\w)\geq \epsilon)\leq 1-\delta$, where $\kappa$ is the condition number of the problem. 

\item Second, we develop an improved stochastic algorithm by combining robust mean estimation technique and a reference truncation technique.  It enjoys a sample complexity in the order of $O(\frac{\log^2(\log(1/\epsilon)/\delta)}{\mu\epsilon})$ for achieving $\Pr(F(\bar\w) - \min_{\w}F(\w)\geq \epsilon)\leq 1-\delta$  for $\mu$-optimal strongly convex function $F(\w)$. This sample complexity matches the lower bound of stochastic strongly convex optimization not only in terms of $\epsilon$ but also in terms of $\mu$ up to a marginal logarithmic factor. 

\end{itemize}  To the best of our knowledge, this is the first work that establishes nearly optimal  sub-Gaussian confidence bounds for compositional problems under heavy-tailed assumptions.

\section{Related Work}
{\bf Stochastic Compositional Problems.} Research on stochastic compositional problems dates back to 70s by Ermoliev~\cite{Ermoliev76}. The first non-asymptotic convergence analysis was given by Wang et al. in 2014~\cite{DBLP:journals/mp/WangFL17}. They considered a broader family of problems in the form of $\E_v[f_{v}(\E_{\xi}g_\xi(\w))]$, where both $f$ and $g$ are random functions. Under the same conditions of this paper and Lipschitz continuity and bounded second-order moment of $\nabla f_v(\cdot)$, the authors established a convergence rate of $O(1/T^{1/4})$ with $T$ iterations of updates  for a general convex objective in terms of the objective gap,  and $O(1/(\mu^{8/3}T^{2/3}))$ for $\mu$-strongly convex objective in terms of distance of solution to the optimal set. Applying to our problem, it implies a sample complexity of $O(1/(\mu^4\epsilon^{1.5}))$ for making the objective gap $\epsilon$ small. When the objective function is smooth, the authors also developed a faster convergence rate in order of $O(1/(\mu^{14/5}T^{4/5}))$ under a stronger assumption about $\nabla g_\xi(\w)$, i.e., the fourth-order moment is bounded. They also consider non-convex smooth objectives, which is out of scope of this paper. 

A series of extensions and improvements were made in later works. Wang et al.~\cite{DBLP:journals/jmlr/WangLF17} extended their algorithms to solving additive compositional problem of the form $\E_v[f_{v}(\E_{\xi}g_\xi(\w))] + r(\w)$ where $r$ is a non-smooth convex regularizer. They improved the rates to $O(1/T^{2/7})$ for a general convex objective and $O(1/T^{4/5})$ for an optimal strongly convex objective.  They also achieved  a sample complexity  of $O(1/\epsilon)$ when $f$ is linear or $g$ is linear.  However, in terms of assumptions on random function $g_\xi(\w)$ they simply assume that its gradient is bounded. In contrast, this paper considered that $f$ is deterministic and established fast rates as good as $O(1/\epsilon)$ under much weaker assumptions.  Yang et al.~\cite{DBLP:journals/siamjo/YangWF19} considered multi-level stochastic compositional problem under  similar bounded moment assumptions as made in~\cite{DBLP:journals/mp/WangFL17} and established generic convergence rate dependent on the stochasticity level. Wang and Liu~\cite{DBLP:conf/wsc/WangL16} also considered Markov noise in estimating $g$, $\nabla g$ and $\nabla f$, and established a rate of $O(1/T^{1/6})$ for general convex objectives under bounded noise assumptions.

Improved rates were established by using variance reduction techniques (e.g., SVRG~\cite{NIPS2013_4937}, SCCG~\cite{DBLP:conf/aistats/LeiJ17}) for finite-sum problems~\cite{DBLP:conf/aistats/LianWL17,DBLP:journals/corr/abs-1806-00458,DBLP:journals/tnn/LiuLT19,DBLP:conf/aaai/HuoGLH18,DBLP:journals/corr/abs-1809-02505,pmlr-v97-zhang19n,DBLP:conf/nips/ZhangX19,DBLP:conf/ijcai/YuH17}. These works established linear convergence for finite-sum problems with strongly convex objectives and improved sub-linear convergence for general objectives. In these works, they simply assume the inner function $g_\xi(\cdot)$ has bounded and Lipschitz continuous  gradient. Some of these works also considered online problems without a finite-sum structure, and a sample complexity  in the order of $O(1/\epsilon)$ was also achieved in some of these works for (optimal) strongly convex objectives. For example, Liu et al.~\cite{DBLP:journals/corr/abs-1809-02505} considered $\E_v[f_{v}(\E_{\xi}g_\xi(\w))]$ and their algorithms enjoy a sample complexity of $O(1/(\mu^2\epsilon))$ for finding a solution $\bar\x$ such that $\E[\|\bar\x- \x_*\|^2]\leq \epsilon$. Zhang \& Xiao~\cite{DBLP:conf/nips/ZhangX19} considered the same objective  $f(\E_\xi g(\cdot; \xi)) + r(\cdot)$  as in this work and  and their algorithms enjoy a sample complexity of $O(\log(1/\epsilon)/(\mu\epsilon))$ for finding a solution $\bar\x$ such that $\E[F(\bar\x) - F(\x_*)]\leq \epsilon$ when $F$ is $\mu$-optimal strongly convex. Nevertheless, these works require each random function $g_\xi(\cdot)$ to be smooth or at least smooth in expectation, i.e., $\E_\xi\|\nabla g_\xi(\w_1) - \nabla g_\xi(\w_2)\|^2\leq O(\|\w_1 - \w_2\|^2)$. In contrast, we avoid making this assumption and enjoy a similar sample complexity in the order of $O(\log(\log(1/\epsilon))/(\mu\epsilon))$ with high probability.

{\bf Robust Stochastic Methods.} Many studies have established sub-Gaussian  convergence bounds for stochastic convex optimization~\cite{Nemirovski:2009:RSA:1654243.1654247,DBLP:journals/jmlr/HazanK11a,RakhlinSS12,DBLP:journals/corr/abs-1607-01027,DBLP:journals/siamjo/GhadimiL13}. But most of them assume the noise is light-tailed (e.g., bounded or sub-Gaussian). Recently, Davis et al.~\cite{DBLP:journals/corr/abs/1907.13307} proposed a generic boosting technique to boost any low convergence results that can be achieved under heavy-tailed noise into  high probability convergence for strongly convex objectives. It is unclear how strong convexity can be alleviated into optimal strong convexity condition of the objective, which captures a much broader family of problems. Their algorithm is built on the proximal point method and robust distance estimation technique~\cite{hsu2016loss}. Nazin et al.~\cite{DBLP:journals/aarc/NazinNTJ19} proposed robust mirror descent method for solving stochastic convex optimization problems based on truncating stochastic gradient. They also considered optimal strongly convex objective and established a fast rate of $O(\log(\log(\kappa)/\delta)/(\mu\epsilon))$. However, we notice that they assumed the optimal set is in the interior of the constrained domain.  We avoid such condition in order to capture a broader family of problems. There also exists many studies focusing on sub-Gaussian confidence bounds for empirical loss minimization and bandits problems under heavy-tailed assumption~\cite{hsu2014heavy,hsu2016loss,audibert2011robust,DBLP:conf/uai/XuZYZJY19,DBLP:conf/nips/ZhangZ18,DBLP:conf/icml/LuWHZ19,DBLP:journals/tit/BubeckCL13}, which is out of scope of this work.

\section{Preliminaries}
We denote by $\|\cdot\|$ the Euclidean norm of a vector or the Frobenius norm of a matrix. Define $\mathcal B(\c, r)=\{\x\in\R^p: \|\x - \c\|\leq r\}$ an Euclidean ball centered at $\c$ with a radius $r$. A function $h(\w)$ is $\mu$-strongly convex if for any $\w, \w'$ it holds $h(\w)\geq h(\w') + \partial h(\w')(\w- \w') + \frac{\mu}{2}\|\w - \w'\|^2$. A function $h(\w)$ is said to be $\mu$-optimal  strongly convex if for any $\w$ it holds $h(\w) - \min_\w h(\w)\geq \frac{\mu}{2}\|\w - \w_*\|^2$, where $\w_*$ is an optimal solution (not necessarily unique) that is closest to $\w$. This condition is also known as the quadratic growth condition in some literature~\cite{DBLP:journals/corr/abs-1607-01027}. It is much weaker than the strong convexity condition.

Including the assumptions in~(\ref{eqn:ass}), we make the following assumptions throughout the paper.
\begin{assumption}
\label{assumption:composite_problem}
Let $C_f$, $L_f$, $C_g$, $L_g$, $\sigma_1$ and $\sigma_0$ be positive scalars.
\begin{itemize}[leftmargin=*]
\item (i): $f(\cdot)$ is $C_f$-Lipschitz continuous and has $L_f$-Lipschitz continuous gradients.

\item (ii): $g(\cdot)$ has $L_g$-Lipschitz continuous gradient.

\item  (iii): $f(g(\cdot))$ and $r(\cdot)$ are convex functions, $F(\w)$ is $\mu$-optimal strongly convex. 
\item  (iv): Regarding the random function $g(\w; \xi)$, we have
\begin{align*}
\E[\| \nabla g(\w; \xi) \|^2] \leq C_g^2,\quad
\E[\| \nabla g(\w; \xi) - \nabla g(\w) \|^2] \leq \sigma_1^2,\quad
\E[\| g(\w; \xi) - g(\w) \|^2] \leq \sigma_0^2.
\end{align*}
\end{itemize}
\end{assumption}
\begin{rem}
The first two assumptions are used to ensure that $f(g(\w))$ is a $L$-smooth function with $L=(C_f L_g + C_g^2 L_g)$, which were also imposed in~\cite{DBLP:journals/mp/WangFL17}. However, the difference from previous works lies at the third assumption regarding the random function $g(\cdot; \xi)$. Many studies impose much stronger assumptions than ours. For example, Wang et al.~\cite{DBLP:journals/mp/WangFL17} assumed that the fourth order moments $\E\|\nabla g(\w; \xi)\|^4$ is bounded in order to derive a faster rate for strongly convex problems in the order of $O(1/\epsilon^{5/4})$. Wang et al.~\cite{DBLP:journals/jmlr/WangLF17} simply assume that $\|\nabla g(\w; \xi)\|$ is bounded. Zhang \& Xiao~\cite{DBLP:conf/nips/ZhangX19} assumed that $g(\w;\xi)$ has bounded and Lipschitz continuous gradient. Please also note that $\E\|\nabla g(\w; \xi)\|^2\leq C_g^2$ indicates that $g(\w)$ is $C_g$-Lipschitz continuous. Below, we let $\kappa = L/\mu$ denote the condition number of the problem. 
\end{rem}

An ingredient of the proposed robust stochastic algorithm is the robust mean estimator, which has been studied in many previous works~\cite{3748,hsu2016loss}. 
Given a set of  $n$ i.i.d random variables $X_1, \ldots, X_n$ from the same distribution with a mean $\mu$ and finite variance $\sigma^2$, the problem is to obtain a robust mean estimator $\bar\mu$ with a sub-Gaussian confidence bound. Depending on whether the random variable is a scalar or a vector, we can use different approaches for computing the robust mean estimator.  If $X_i\in\R$, we can use the simple median-of-means (MoM) estimator~\cite{3748}; otherwise we can use the robust distance approximation method proposed in~\cite{hsu2016loss}. We present both methods in Algorithm~\ref{algorithm:rde} and summarize their key properties below. 

\begin{lem}
\label{lemma:median_of_mean}
Given $n$ i.i.d. random variables $X_1, ..., X_n$ with mean $\mu$ and finite variance $\sigma^2$, let $\delta \in (0, 1)$. If $X_i\in\R$, then Algorithm~\ref{algorithm:rde} with $m = \frac{n}{8 \log(1/\delta)}$ and $k = 8 \log(1/\delta)$ guarantees that with probability at least $1 - \delta$, we have $(\bar \mu- \mu)^2 \leq \frac{ 32 \sigma^2 \log(1/\delta) }{n} $. 
 If $X_i\in\R^p$, then Algorithm~\ref{algorithm:rde} with $m =  \frac{n}{18 \log(1/\delta)}$ and $k = 18 \log(1/\delta)$ guarantees that with probability at least $1 - \delta$, we have $\|\bar \bmu- \bmu\|^2 \leq \frac{ 486 \sigma^2 \log(1/\delta) }{n} $. 
\end{lem}



\begin{algorithm}[t]
\caption{Robust Mean Estimator: RME($\{X_1, \ldots, X_n\}$)}
\label{algorithm:rde}
\begin{algorithmic}[1]
\STATE Initialize $m$ and $k$ such that $mk=n$
\STATE Compute $\bar\mu_1= \frac{1}{m} \sum_{i=1}^m X_i, \ldots, \bar\mu_k = , \frac{1}{m} \sum_{i=(k-1)m+1}^{km} X_i$
\IF {$X_i\in\R$}
\STATE return $\bar\mu= \text{median}(\bar\mu_1, \ldots, \bar\mu_k)$
\ELSE
\STATE Define $\mathcal X=\{\bar\mu_1, \ldots, \bar\mu_k\}$
\STATE For each $i=1,\ldots, k$, compute $r_i = \min\{r\geq 0: |\mathcal B(\bar\mu_i, r)\cap \mathcal X|\geq k/2\}$
\STATE return $\bar \mu = \bar\mu_{i_*}$ where $i_* = \arg\min_{i\in\{1,\ldots, k\}}r_i$
\ENDIF
\end{algorithmic}
\end{algorithm}

\section{A Baseline: From Low Probability  to High Probability Convergence}\label{section:MSCG_non_robust}

In this section, we will present  a logically simple  solution to achieving a sub-Gaussian confidence bound under Assumption~\ref{assumption:composite_problem}. In particular, we will first present an algorithm with an expectational convergence, and then boost it directly into high probability convergence by using robust mean estimators.  

We describe the basic algorithm in Algorithm~\ref{algorithm:MSCG}, which is referred to as mini-batch stochastic compositional gradient (MSCG) method. The key update in Step \ref{algorithm:MSCG:line:update} mimics the well-known stochastic proximal gradient update  except that the gradient estimator computed by $\z_{t+1}^{\top}f(\y_{t+1})$ is a biased stochastic gradient of $f(g(\w))$ at $\w_t$. In order to control the error of the gradient estimator, we use a mini-batch of data to compute an estimation of $g(\w_t)$ and $\nabla g(\w_t)$ by $\y_{t+1}$ and $\z_{t+1}$, respectively. The following lemma summarizes the convergence bound of MSCG. 

\begin{algorithm}[t]
\caption{MSCG$(\w_0, \eta, T, \{m_t\})$}
\label{algorithm:MSCG}
\begin{algorithmic}[1]
\FOR{$t = 0, ..., T$}

\STATE Sample random samples $\calS_1, \calS_2$ of size $m_t$

\STATE (non-robust option)  $\y_{t+1} = \frac{1}{m_t} \sum_{\xi \in \calS_1} g(\w_t; \xi)$ and $\z_{t+1} = \frac{1}{m_t} \sum_{\xi \in \calS_2} \nabla g(\w; \xi)$
\STATE (robust option):  $\y_{t+1}=\text{RME}(\{g(\w_t; \xi), \xi \in\calS_1\})$ and $\z_{t+1} = \text{RME}(\{ \nabla g(\w_t; \xi), \xi\in\calS_2\})$

\STATE $\w_{t+1} = \arg\min_{\w} \z_{t+1}^{\top}\nabla f(\y_{t+1})\w + \frac{1}{2 \eta_t} \| \w - \w_t \|^2 + r(\w)$
\label{algorithm:MSCG:line:update}

\ENDFOR

\STATE {\bf Output:} $\wh_T = \calO(\w_1, ..., \w_T)$
\end{algorithmic}
\end{algorithm}

\begin{algorithm}[t]
\caption{RMSCG}\label{algorithm:RMSCG}
\begin{algorithmic}
\STATE Initialize $\w_0, \eta, m_1, T$

\FOR{$k=1, ..., K$}

\STATE $\w_k =$ MSCG$(\w_{k-1}, \eta, T, m_k)$

\STATE Update $m_{k+1} = 2m_k$

\ENDFOR
\end{algorithmic}
\end{algorithm}

\begin{lem}\label{lemma:MSCG_convergence}
Suppose Assumption~\ref{assumption:composite_problem} holds. 
Let $F(\w_0) - F(\w_*) \leq \epsilon_0$.  For Algorithm \ref{algorithm:MSCG} with non-robust option and $\eta \leq \frac{1}{2L},  \wh_T = \frac{1}{T} \sum_{t=1}^T \w_{t}$,   we have
\begin{align*}
&\E[F(\wh_T) - F(\w_*)]\leq \frac{\epsilon_0}{T} +\frac{1}{T} \sum_{t=0}^{T-1}\frac{ 4\eta (C_f^2\sigma_1^2+ C_g^2 L_f^2\sigma_0^2) + 2L_f^2C_g^2\sigma_0^2/\mu}{m_t}  + \frac{\|\w_0 - \w_*\|^2}{\eta T}.
\end{align*}
\end{lem}
{\bf Remark:} By setting $m_t \propto (t+1)/\mu $, the above result implies $\E[F(\wh_T) - F(\w_*)]\leq O(\frac{\log T}{T})$, which yields a sample complexity of $\sum_{t=1}^Tm_t = \widetilde O(\frac{1}{\mu\epsilon^2})$ in order to have $\E[F(\wh_T) - F(\w_*)]\leq\epsilon$. 

Nevertheless, the above result can be improved by using the restarting trick.  
The algorithm is presented in Algorithm~\ref{algorithm:RMSCG}. Its convergence is summarized below.

\begin{thm}
\label{theorem:restart_MSCG}
Suppose  Assumption~\ref{assumption:composite_problem} holds. 
Let $F(\w_0) - F(\w_*) \leq \epsilon_0$ and $\epsilon_k = \epsilon_0/2^k$.
For Algorithm \ref{algorithm:RMSCG} with non-robust option, let $\eta \leq \frac{1}{2L}$, $T=4/(\mu\eta)$, $m_k = 4 ( \mu \eta C_f^2 \sigma_1^2 + \mu \eta C_g^2 L_f^2 \sigma_0^2 + C_g^2 L_f^2 \sigma_0^2 ) / ( \mu\epsilon_{k-1} )$ and $\wh_T = \sum_{t=1}^T \w_{t} / T$.
{The sample complexity for achieving $\E[F(\wh_T) - F(\w_*)] \leq \epsilon$ is given by  $m_{tot} = \tilde O(\frac{1}{\mu^{2}\epsilon})$.}
\end{thm}

{\bf Remark: }We make a comparison with the result established in~\cite{DBLP:journals/mp/WangFL17} under the same condition of Theorem~\ref{theorem:restart_MSCG}. They proved a result for $\E[\|\w_T - \w_*\|^2]\leq \epsilon$ with a sample complexity of $O(\frac{1}{\mu^3\epsilon^{3/2}})$. When $r=0$, their result together with smoothness of $F$ implies that their algorithm has a sample complexity of $O(\frac{1}{\mu^4\epsilon^{3/2}})$ for achieving $\E[F(\w)- F(\w_*)]\leq \epsilon$. It is clear that our result in Theorem~\ref{theorem:restart_MSCG} is much better.

Next, we discuss how to boost the expectational convergence result into a high probability result without using the complicated boosting technique introduced by~\cite{DBLP:journals/corr/abs/1907.13307}. 
To this end, we present the following lemma for explaining the robust algorithms. 
\begin{lem}\label{lemma:RMSCG_one_stage_0}
Suppose Assumption \ref{assumption:composite_problem} holds.
Let $\eta_t \leq \frac{1}{2L}$.
After running $T$ iterations of Algorithm \ref{algorithm:MSCG}, we have
\begin{align}
\label{eq:RMSCG_one_stage_0}
\frac{1}{T}\sum_{t=0}^{T-1}(F(\w_{t+1}) - F(\w_*))
\leq &
\frac{\| \w_* - \w_0 \|^2}{2\eta T}
+ \frac{ \eta }{T} \underbrace{\sum_{t=0}^{T-1} \| \nabla g(\w_t)^{\top} \nabla f(g(\w_t)) - \z_{t+1} ^{\top}\nabla f(\y_{t+1}) \|^2}_{A}
\nonumber\\
&
+ \frac{1}{T} \underbrace{\sum_{t=0}^{T-1} \langle  \nabla g(\w_t)^{\top}\nabla f(g(\w_t)) -  \z_{t+1}^{\top}\nabla f(\y_{t+1}), \w_t - \w_*\rangle}_{B}  
\end{align}
\end{lem}The first term in the upper bound
 can be handled by optimal strongly convexity by setting $\w_*$ as the closest optimal solution to $\w_0$ to build a recursion, i.e., $\| \w - \w_* \|^2\leq \frac{2}{\mu}(F(\w_0) - F(\w_*))$. The challenge lies at bounding the second term $A$ and the third term $B$ with sub-Gaussian confidence bound.  Let us consider how to bound $A$ with high probability. In order to do so, we decompose $A, B$ into two terms: 
\begin{align*}
&A/2\leq\\
&  \underbrace{ \sum_{t=0}^{T-1} \|\nabla g(\w_t) ^{\top}\nabla f(g(\w_t))-  \nabla g(\w_t)^{\top}\nabla f(\y_{t+1})\|^2}_{A_1} + \underbrace{\sum_{t=0}^{T-1} \|\nabla g(\w_t)^{\top} \nabla f(\y_{t+1})  - \z_{t+1}^{\top} \nabla f(\y_{t+1}) \|^2}_{A_2}\\
&B\leq \underbrace{\sum_{t=0}^{T-1} \langle  \nabla g(\w_t)^{\top}\nabla f(g(\w_t)) -  \nabla g(\w_t)^{\top}\nabla f(\y_{t+1}), \w_t - \w_*\rangle}_{B_1}\\&\hspace*{0.1in}+\underbrace{\sum_{t=0}^{T-1} \langle  \nabla g(\w_t)^{\top}\nabla f(\y_{t+1}) -  \z_{t+1}^{\top}\nabla f(\y_{t+1}), \w_t - \w_*\rangle}_{B_2}
\end{align*}
The first term $A_1$ can be bounded by $C_g^2L_f^2\sum_{t=0}^{T-1}\|g(\w_t) - \y_{t+1}\|^2$ by the Lipschitz continuity of $g$ and $\nabla f$. The second term $A_2$ can be bounded by $C_f^2\sum_{t=0}^{T-1}\|\nabla g(\w_t) - \z_{t+1}\|^2$. Similarly, $B$ can be bounded by $B\leq \sum_{t=0}^{T-1}C_g L_f\|g(\w_t) - \y_{t+1}\|\|\w_t - \w_*\| + C_f\sum_{t=0}^{T-1}\|\nabla g(\w_t) - \z_{t+1}\|\|\w_t - \w_*\|$. Hence, in order to bound $A$ and $B$ with high probability, we need to robustify the mean estimator of $g(\w_t)$ and $\nabla g(\w_t)$. A simple approach is to replace $\y_{t+1}$ and $\z_{t+1}$ by robust mean estimator computed by Algorithm~\ref{algorithm:rde}, i.e., using the robust option in Algorithm~\ref{algorithm:MSCG}.  Then following the similar analysis to the proof of Theorem~\ref{theorem:restart_MSCG}, we can establish high probability convergence result with a sample complexity of $O(\frac{\log(\kappa/\delta)}{\mu^2\epsilon})$, where $\log(\kappa/\delta)$ comes from that  applying the union bound over $T=O(\kappa)$  calls of  robust mean estimators.


\begin{thm}
\label{theorem:robust_restart_MSCG}
Suppose Assumption~\ref{assumption:composite_problem} holds. 
Let $F(\w_0) - F(\w_*) \leq \epsilon_0$, 
$\epsilon_k = \epsilon_0/2^k$, and 
$K=\log(\epsilon_0/\epsilon)$. 
For Algorithm \ref{algorithm:RMSCG} with robust option in Algorithm~\ref{algorithm:MSCG}, let $\eta \leq \frac{1}{2L}$, 
$T = 4 ( \frac{2}{\mu \eta} + 1 )$,
$m_k = \frac{16}{\mu \epsilon_{k-1}} ( \mu \eta + 1 ) ( C_g^2 L_f^2 \sigma_0^2 + C_f^2 \sigma_1^2 ) 486 \log(1/\delta)$
and $\wh_T = \sum_{t=1}^T \w_{t} / T$. Then with probability $1-2TK\delta$, we have $F(\w_K) - F(\w_*) \leq \epsilon$. 
\end{thm}
{\bf Remark: } By making the high probability to be $1-\delta$, the sample complexity becomes $O(\frac{\log(\kappa\log(1/\epsilon)))}{\mu^2\epsilon})$.



\section{Nearly Optimal  High Probability Convergence under Heavy-tailed Noise}\label{section:RMSCG_ref_robust}
Although the above approach that  replaces $\y_{t+1}$ and $\z_{t+1}$ by their robust counterparts can help us derive a sub-Gaussian confidence bound, it is sub-optimal for stochastic strongly convex optimization. When $f$ is a linear function the problem reduces to a stochastic convex optimization, whose lower bound complexity  under $\mu$-strong convexity is $O(1/(\mu\epsilon))$~\cite{DBLP:journals/jmlr/HazanK11a}.   {\it Can we match such lower bound or prove that the stochastic convex compositional problem under optimal strong convexity is harder?} 
In this section, we present an affirmative answer to this question. 

Below we present a better approach that enjoys a nearly optimal sub-Gaussian confidence bound in the order of $O(\frac{\log(\log(1/\epsilon)/\delta)}{\mu\epsilon})$. Our key idea is to make use of another robust technique, i.e., truncation~\cite{DBLP:journals/aarc/NazinNTJ19}. In particular, we will truncate the mean estimators $\y_{t+1}$ and $\z_{t+1}$ when their magnitude is sufficiently large. 

The detailed steps of the proposed algorithm are presented in Algorithm~\ref{algorithm:RROSC} and Algorithm~\ref{algorithm:ROSC}.  The main algorithm (referred to as RROSC) is also a restarted version of a basic algorithm (referred to as ROSC). There are two differences between ROSC and  MSCG. First, the mean estimator $\y_{t+1}$ and $\z_{t+1}$ are replaced by their truncated version for updating $\w_t$, which are defined by
\begin{align}\label{eq:RMSCG_ref_truncation_y}
\yh_{t+1}
=
\left\{
\begin{array}{ll}
    \y_{t+1}, & \text{ if } \| \y_{t+1} - \widetilde\y_0 \| \leq C_g \| \w_t -  \w_0 \| + \nu \sigma_0 + \lambda \\
    \widetilde{\y}_0,     & \text{ otherwise , } 
\end{array}
\right.
\end{align}
and
\begin{align}\label{eq:RMSCG_ref_truncation_g}
 \zh_{t+1}
=
\left\{
\begin{array}{ll}
    \z_{t+1}, & \text{ if } \| \z_{t+1} - \widetilde \z_0 \| \leq L_g\| \w_t -\w_0 \| + \nu \sigma_1 + \lambda \\
    \widetilde \z_0,     & \text{ otherwise , } 
\end{array}
\right.
\end{align}
where $\widetilde\y_0$ and $\widetilde\z_0$ is the robust mean estimator for $g(\cdot)$ and $\nabla g(\cdot)$ at the initial  point $\w_0$ of each restart, such that $\| \widetilde\y_0 - g(\w_0) \| \leq \nu \sigma_0$  and  $\| \widetilde\z_0 - \nabla g(\w_0) \| \leq \nu \sigma_1$ hold with high probability $1-2\delta$ for some constant $\nu>0$, and $\lambda$ is an appropriate number, which will be given in the theorem. The second difference is that for updating $\w_{t+1}$ we explicitly add a bounded ball constraint $\|\w- \w_0\|\leq D$ and shrink the radius $D$ after each restart. 

We notice that the truncation technique and the shrinking bounded ball trick have been used in~\cite{DBLP:journals/aarc/NazinNTJ19} for deriving sub-Gaussian confidence bound for stochastic convex optimization under heavy-tailed noise. Nevertheless, we would like to emphasize the novelty of Algorithm~\ref{algorithm:RROSC}. In~\cite{DBLP:journals/aarc/NazinNTJ19}, for non-compositional optimal strongly convex function the authors consider a constrained problem and assume that the optimal solution lies the interior of the constrained domain. Hence, they truncated the gradient to zero when its magnitude is sufficiently large. In contrast, we avoid such restriction in order to capture a wide range of non-smooth regularizer $r$ including an indicator function of a constraint. Correspondingly, we truncate the estimators $\y_{t+1}$ and $\z_{t+1}$ to the robust estimators of the initial solution $\w_0$, which is denoted by {reference truncation}.  Together with the shrinking ball trick, the truncated versions of $\y_{t+1}$ and $\z_{t+1}$ are not only bounded but also have bounded variance. 


By using the reference truncation technique, we can leverage advanced concentration inequality, in particular Bernstein inequality to bound $A$ and $B$ of Lemma~\ref{lemma:RMSCG_one_stage_0} directly instead of bounding each individual terms in $A$ and $B$ separately. As a reseult, we can aovid a logarithmic dependence on the condition number.

The following theorem provides the improved convergence guarantee for Algorithm~\ref{algorithm:RROSC} by setting $\eta_k, T_k, D_k, \lambda_k, m_k$ appropriately. 
\begin{algorithm}[t]
\caption{RROSC$(\w_0, \eta_1,  T_1, m)$}\label{algorithm:RROSC}
\begin{algorithmic}[1]
\STATE Initialize $D_1$

\FOR{$k=1, ..., K$}

\STATE $\w_k =$ ROSC$(\w_{k-1}, \eta_k, T_k, D_k, \lambda_k, m)$

\STATE Update $\eta_k,  T_k, D_k, \lambda_k$

\ENDFOR
\end{algorithmic}
\end{algorithm}
\begin{algorithm}[t]
\caption{ROSC$(\w_0, \eta, T, D, \lambda, m)$}\label{algorithm:ROSC}
\begin{algorithmic}[1]
\STATE Sample random samples $\calS$ of size $b$
\STATE Compute  $\widetilde\y_0 = \text{RME}(\{g(\w_0; \xi), \xi \in\calS\})$ and $\widetilde \z_0 = \text{RME}(\{ \nabla g(\w_0; \xi), 
\xi \in \calS \})$.
\label{algorithm:robust_MSCG:line:reference}

\FOR{$t = 0, ..., T$}

\STATE Sample random samples $\calS$ of size $m$.

\STATE Compute $\y_{t+1} = 1/m\sum_{\xi \in \calS} g(\w_t; \xi)$ and $\yh_{t+1}$ according to (\ref{eq:RMSCG_ref_truncation_y}).

\STATE Compute $\z_{t+1} = 1/m\sum_{\xi \in \calS} \nabla g(\w_t; \xi)$ and $\zh_{t+1}$ according to (\ref{eq:RMSCG_ref_truncation_g}).

\label{algorithm:ROSC:line:update_w}
\STATE $\w_{t+1} = \min_{\| \w - \w_0 \| \leq D} \langle  \zh_{t+1}^{\top}\nabla f(\yh_{t+1}), \w \rangle + \frac{1}{2\eta} \| \w - \w_t \|^2 + r(\w)$.

\ENDFOR
\end{algorithmic}
\end{algorithm}

\begin{thm}\label{theorem:RROSC_ref_convergence} 
Suppose Assumption \ref{assumption:composite_problem} holds.
Let $F(\w_0) - F(\w_*) \leq \epsilon_0$ and $\epsilon_k = \epsilon_0 / 2^k$. For any $\epsilon\in(0,\epsilon_0)$, let $K=\log(\epsilon_0/\epsilon)$. 
For Algorithm \ref{algorithm:RROSC}, by setting $b=18c\log(1/\delta))$ where $c\in \mathbf N^+$,  
$\eta_k = O(\epsilon_{k-1})\leq 1/(2L), 
T_k = O(\max(\log^2(1/\delta)/(\mu\epsilon_{k-1}), \log(1/\delta)/b)), 
D_k=\sqrt{\frac{2\epsilon_{k-1}}{\mu}}$, 
$\lambda_k = O(\max( \sqrt{ T_k/m }, D_k )) $, 
then with probability at least $1 - 6K\delta$, we have $F(\w_K) - F(\w_*)\leq \epsilon$.
 The sample complexity is given by $\sum_{k=1}^KmT_k=O(\log^2(1/\delta)/(\mu\epsilon))$. 
\end{thm}
{\bf Remark: } Different from Theorem~\ref{theorem:restart_MSCG}, we can set the number of samples for computing the estimators of $g$ and $\nabla g$ as a constant, which gives  $\eta_k$ a geometric decreasing sequence and $T_k$ a geometric increasing sequence. 
It is notable that this complexity is nearly optimal up to a logarithmic factor  as it includes stochastic strongly convex optimization as a special case, whose lower bound is proved to be $O(1/(\mu\epsilon))$~\cite{DBLP:journals/jmlr/HazanK11a}.

\subsection{Analysis: Proof Sketch}
We present a proof sketch in this subsection  for proving the main result. 
Lemma~\ref{lemma:RMSCG_one_stage_0} is the starting point for proving our result. Different from previous section, we will bound $A$ and $B$ in a whole instead of bounding each term in the summation individually. To this end, we can decompose $A$ and $B$ into two terms and bound the two terms separately. Below, we denote $\hat \sigma_m^2 = \sigma_1^2 / m$ and $\tilde \sigma_m^2 = \sigma_0^2 / m$. 
\begin{lem}\label{lemma:truncation_bounds}
Let us consider Algorithm~\ref{algorithm:RROSC}. 
Given $\w_0$ suppose $\| \widetilde\y_0 - g(\w_0) \| \leq \nu \sigma_0$  and  $\| \widetilde\z_0 - \nabla g(\w_0) \| \leq \nu \sigma_1$ hold with high probability $1-2\delta$, where $\nu \leq \sqrt{T}$, and $\lambda = O(\max( \sqrt{ T / m }, D))$, then for any $\w$ such that $\| \w_0 - \w \| \leq D$, with probability $1-4\delta$
\begin{align*}
A_1=: &
\sum_{t=0}^{T-1} \| \nabla g(\w_t)^\top \nabla f(g(\w_t)) - \nabla g(\w_t)^\top \nabla f(\yh_{t+1}) \|^2 
\leq 
8 C_g^2 L_f^2 ( 3\log(1/\delta) + 2 ) \max( \tilde\sigma_m^2 T, C_g^2 D^2 ) \\
A_2:= &
\sum_{t=0}^{T-1} \| \nabla g(\w_t)^\top \nabla f(\yh_{t+1}) - \zh_{t+1}^\top \nabla f(\yh_{t+1}) \|^2 
\leq 8 C_f^2 ( 3\log(1/\delta) + 2 ) \max( \hat\sigma_m^2 T, L_g^2 D^2 ) 
\\
B_1:= &
\sum_{t=0}^{T-1} \langle \nabla g(\w_t)^\top \nabla f(g(\w_t)) - \nabla g(\w_t)^\top \nabla f(\yh_{t+1}) , \w_t - \w \rangle 
\\
& \qquad \qquad \qquad \qquad \qquad \qquad \qquad \qquad \qquad 
\leq 
8 C_g L_f D ( \log(1/\delta) + 1 ) \max( \tilde\sigma_m \sqrt{T}, C_g D )
\\
B_2:= &
\sum_{t=0}^{T-1} \langle \nabla g(\w_t)^\top \nabla f(\yh_{t+1}) - \zh_{t+1}^\top \nabla f(\yh_{t+1}) , \w_t - \w \rangle 
\leq 
8 C_f D ( \log(1/\delta) + 1 ) \max(\hat\sigma_m \sqrt{T}, L_g D).
\end{align*}
\end{lem}

The above results are proved by using Bernstein inequality for martingales~\cite{peel2013empirical}. To this end, truncations (\ref{eq:RMSCG_ref_truncation_y}) and (\ref{eq:RMSCG_ref_truncation_g}) builds $\zh_{t+1}$ and $\yh_{t+1}$ that allow $\| \zh_{t+1} - \nabla g(\w_t) \|$ and $\| \yh_{t+1} - g(\w_t) \|$ to have the bounded values, expectation and variance.
With these bounds, Bernstein inequality gives sub-Gaussian confidence results.

\begin{proof} (of Theorem \ref{theorem:RROSC_ref_convergence}). 

We use induction to prove the result. Let us first consider $k=1$. It is clear that $\|\w_0- \w_*\|\leq\sqrt{\frac{2}{\mu}(F(\w_0) - F(\w_*))}\leq D_1$.  
By replacing $\y_{t+1}$ and $\z_{t+1}$ with $\yh_{t+1}$ and $\zh_{t+1}$  in (\ref{eq:RMSCG_one_stage_0}) and plugging the results of Lemma \ref{lemma:truncation_bounds}, the following inequality holds with probability $1-6\delta$.
\begin{align}\label{eq:RMSCG_one_stage_robust}
&
F(\wh_T) - F(\w_*)
\leq 
\frac{ \epsilon_0 }{ \mu \eta T}
+ \frac{ 2 \eta }{ T } \cdot 8 C_g^2 L_f^2 ( 3 \log(1/\delta) + 2 ) \max(\sigma_0^2 T / m, C_g^2 D^2)
\nonumber\\
&
+ \frac{ 16 \eta C_f^2 ( 3\log(1/\delta) + 2 )  }{ T } \max( \sigma_1^2 T / m, L_g^2 D^2 )
+ \frac{ 8 C_g L_f D ( \log(1/\delta) + 1 )  }{ T }\max( \sigma_0 \sqrt{T / m}, C_g D )
\nonumber\\
&
+ \frac{ 1 }{ T } \cdot 8 C_f D ( \log(1/\delta) + 1 ) \max( \sigma_1 \sqrt{T / m}, L_g D )  ,
\end{align}
where we apply $\mu$-optimal strong convexity and $F(\w_0) - F(\w_*) \leq \epsilon_0$.

In order to have $F(\wh_T) - F(\w_*) \leq \epsilon_0 / 2$, we can set $m = O(1)$ as a constant and set $T = \frac{10}{\mu \eta}$ and  $\eta = O(\epsilon_0)$ (the detailed value of $\eta$ can be found in the supplement).
Then by induction, after $K = \lceil \log(\epsilon_0 / \epsilon) \rceil$ repeated calls of Algorithm~\ref{algorithm:ROSC} with probability $1 - 6 K \delta$, $F(\w_K) - F_*\leq \epsilon$. As a result, the the sample complexity is 
\begin{align*}
m_{tot}
&= \sum_{k=1}^K mT_k
= \sum_{k=1}^K O\bigg( \frac{ 1}{\mu \eta_k}\bigg )
= \sum_{k=1}^K O\bigg( \frac{ \log^2(1/\delta) }{ \mu \epsilon_k}\bigg )
= O\bigg(\frac{\log^2(1/\delta)}{\mu \epsilon}\bigg)  ,
\end{align*}
where $\eta_k = O(\epsilon_{k-1}/\log^2(1/\delta))$ and $\epsilon_k = \epsilon_{k-1} / 2$.
\end{proof}

\section{Experiments}

In this section, we consider an application of the proposed algorithm in machine learning and present some experimental results. Let us consider the problem of robust learning from multiple distributions~\cite{NIPS2017_7056,DBLP:conf/aaai/QianZTJSL19}.  In particular, suppose there are $m$ data sources with $\P_i, i=1,\ldots, m$ denoting the data distribution of each source, which are not necessarily identical. Denote by $\mathcal L_i(\w) = \E_{\xi\sim \P_i}\ell(\w; \xi)$ the expected loss of data from the $i$-th source, where $\xi$ denotes a random data. Taking into account the inconsistency between $\P_i$ and out of sample distribution, we can formulate a distributionally robust optimization  problem for learning a predictive model $\w$~\cite{NIPS2017_7056,DBLP:conf/aaai/QianZTJSL19}: 
\begin{align*}
\min_{\w\in \R^d}\max_{\p \in\Delta_m}\sum_{i=1}^m p_i \mathcal L_i(\w) - h(\p) + r(\w)
\end{align*}
where $\Delta_m$ is an $m$-dimensional simplex and $h(\p)$ is a convex regularizer of $\p$. By choosing a KL divergence regularization for $\p$, i.e., $h(\p) =\lambda KL(\p, \mathbf 1/m)= \lambda\sum_i p_i \log(mp_i)$, the above problem is equivalent to the following compositional problem: 
\begin{align}\label{eqn:dros}
\min_{\w\in\R^d}\lambda \log\left(\sum_{i=1}^m\exp\left(\frac{\E_{\xi\sim\P_i}[\ell(\w; \xi)]}{\lambda}\right)\right) + r(\w)
\end{align}
which is an instance of~(\ref{eqn:op}) by setting $f(\u) = \lambda \log(\sum_{i=1}^m \exp(u_i/m)): \R^m\rightarrow\R$ and $\E_{\xi}[g(\w; \xi)] = (\E_{\xi\sim\P_1}[\ell(\w; \xi)], \ldots, \E_{\xi\sim\P_m}[\ell(\w; \xi)])^{\top}:\R^d\rightarrow\R^m$. For our experiments, we will consider square loss $\ell(\w; \xi) = (\w^{\top}\x - y)^2$ for a feature-label pair $\xi = (\x, y)$. 
The following lemma shows that~(\ref{eqn:dros}) is a $\mu$-optimal strongly convex function. 
\begin{lem}\label{lemma:square_loss_is_scvx}
Suppose $E_{\xi\sim \mathbb{P}_i} [\x \x^T] \neq 0$,   $\ell(\w; \xi) = (\w^{\top}\x - y)^2$  is a square loss and the epigraph of $r(\w)$ is a polyhedral, the objective in~(\ref{eqn:dros}) is a $\mu$-optimal strongly convex function for some $\mu>0$. 
\label{obj_mu_strongly_convex}
\end{lem}

Next we present our experimental setup and results.
Our main purpose is to demonstrate that our proposed algorithm is robust enough to handle compositional problems under heavy-tailed noise.
To this end, we compare our proposed RMSCG (non-robust option) and RROSC with two relevant algorithms, i.e., Accelerated Stochastic Compositional Proximal Gradient (ASC-PG) \cite{DBLP:journals/jmlr/WangLF17} and Restarted Composite Incremental Variance Reduction (RCIVR) \cite{DBLP:conf/nips/ZhangX19}, which are also designed for compositional problems.

Our experiments are performed on the E2006 dataset\footnote{\url{https://www.csie.ntu.edu.tw/~cjlin/libsvmtools/datasets/regression.html\#E2006-tfidf}.} with some modification.
In order to produce the heavy-tailed noise, we add different categories of noise $\varepsilon$ on the original dataset, particularly, on the label $y$, as in \cite{DBLP:conf/uai/XuZYZJY19}.
They include
1) Pareto noise: we draw a noise $\varepsilon$ from a Pareto distribution with the tail parameter $1/\beta \in \{ 1.01, 2.01 \}$ and then re-center them to get a zero mean;
2) student-t noise: we draw a noise $\varepsilon$ from a Student's t-distribution with degrees of freedom $1/\beta \in \{ 2.5, 5 \}$;
3) sparse noise: we generate a random sparse vector from $[ -\beta, \beta ]$ with $\beta \in \{ 5, 10 \}$, which is added to the output $\y$ with Gaussian noise.
We then generate the new label $y$ by using the six kinds of noise (three different types of noise and two noise levels, respectively) at random, which makes six various distributions on the dataset.

\begin{figure}
\centering
  {\includegraphics[width=0.48\textwidth]{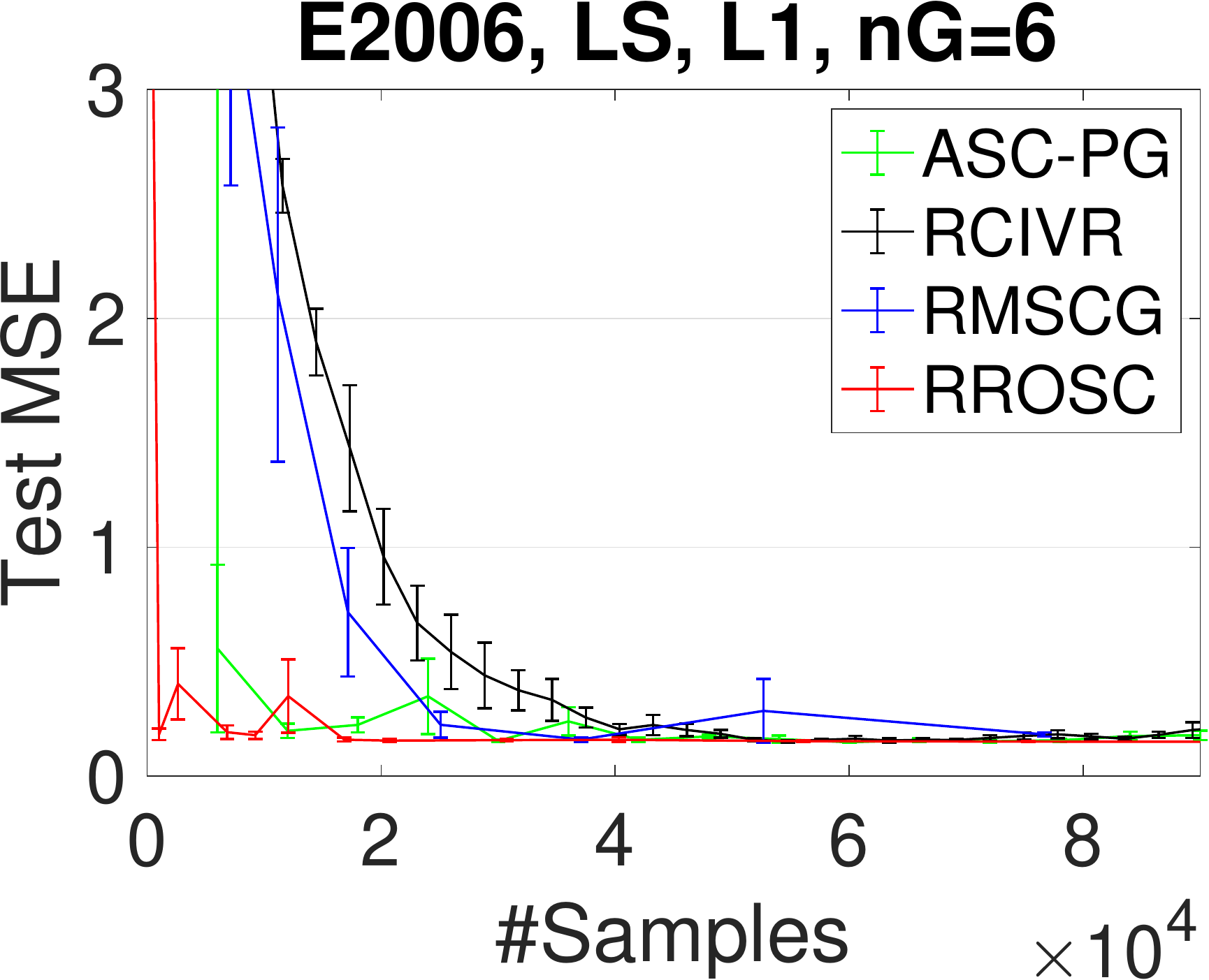}}
  {\includegraphics[width=0.48\textwidth]{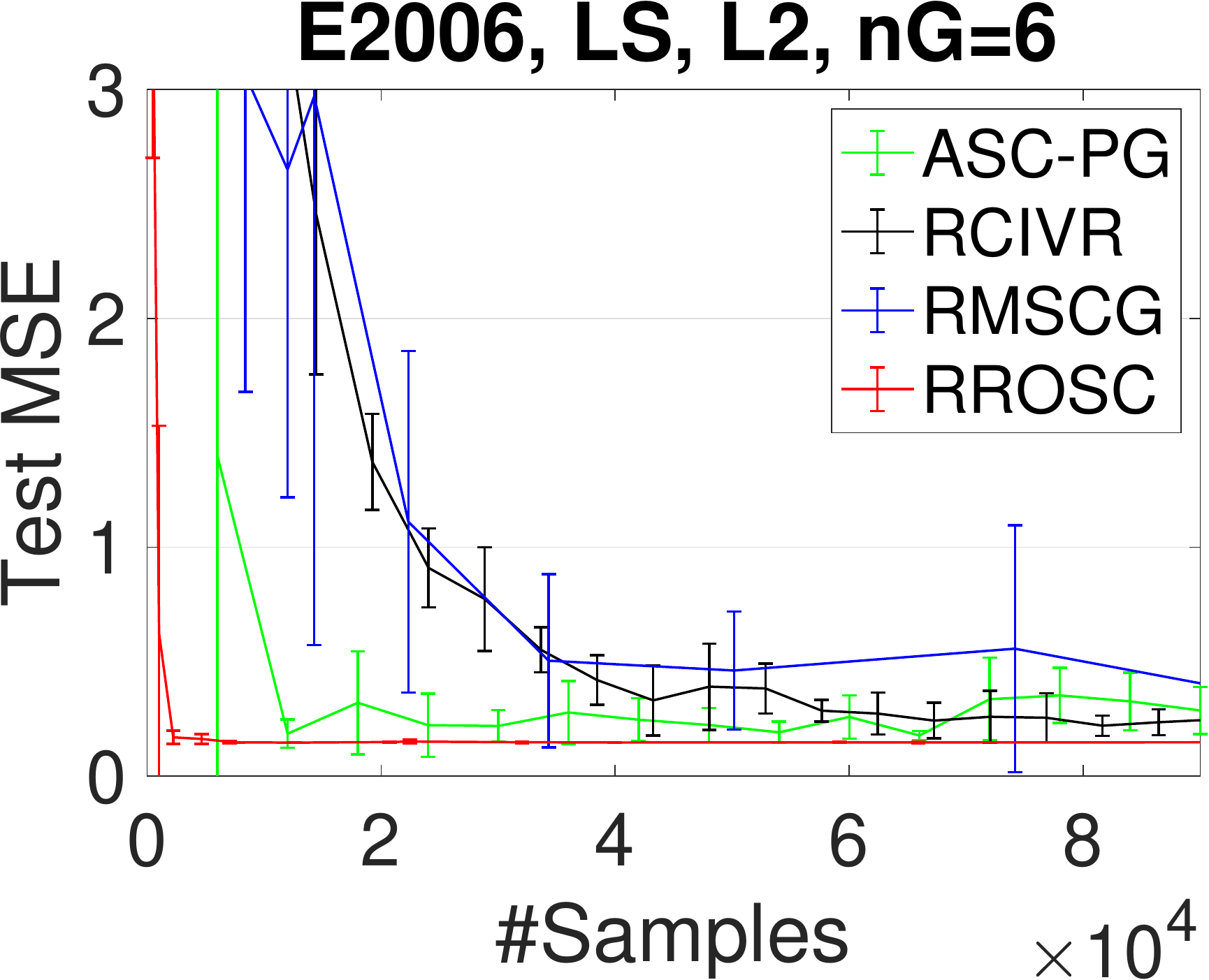}}
\caption{Test MSE error bars: Square loss with $L_1$ regularizer (left) and $L_2$ regularizer (right). X axis represents the number of samples.}
\label{fig:E2006_errerbar_L1_L2}
\end{figure}

To compare the baseline algorithms, we use mean square error (MSE) as the evaluation metric.
We use the original training and testing set for training and testing.
To select the hyper-parameters of the four algorithms and the regularizer parameter $\lambda_{reg}$, we randomly select 10\% of training data as the validation set.
We repeat the experiment for five trials with different random noise and collect the test MSE of each algorithm.
We report the error bar results of four algorithms on two problems in Figure \ref{fig:E2006_errerbar_L1_L2} to show the average MSE with its standard deviation.

For the regularizer $r(\w)$, we employ two commonly used ones, $L_1$ regularizer (i.e., $\lambda_{reg} \| \w \|_1$) and $L_2$ regularizer (i.e., $\frac{ \lambda_{reg} }{ 2 } \| \w \|_2^2$).
We choose the value for $\lambda_{reg}$ from the range $\{ 10^{-5}, 10^{-4}, ..., 10^{5} \}$.
Step size of each algorithm is selected from the range $\{ 10^{-5}, 10^{-4}, ..., 10^{5} \}$ on the validation set.
The radius $D$ of RROSC is set to $100$.

As clearly shown in Figure \ref{fig:E2006_errerbar_L1_L2}, our RROSC is more robust than other baselines.
First, our RROSC achieves faster convergence in terms of test MSE than others.
This supports our theoretical analysis that RROSC has the nearly optimal convergence for stochastic convex compositional problems.
Second, the error bars imply that ASC-PG, RCIVR and RMSCG suffer from the large deviation as they converge, while RROSC enjoys a very stable and robust convergence.
This verifies the effect of the truncation technique and our high probability convergence under heavy-tailed noise.

\section{Conclusion}

We have proposed a single-trial stochastic algorithm for convex compositional problems with heavy-tailed noise.
We employed the truncation technique in our proposed algorithm to achieve sub-Gaussian confidence bounds by Bernstein inequality.
For $\mu$-optimal strongly convex problems, the total sample complexity of our algorithm is $\tilde O(1 / (\mu \epsilon))$.
To the best of our knowledge, this is the first one to establish nearly optimal sub-Gaussian confidence bounds for compositional problems up to a logarithmic factor.

\bibliographystyle{plain}
\bibliography{all,ref}

\newpage
\appendix

\section{Proof of Lemma \ref{lemma:median_of_mean} }
\begin{proof}
The proof of lemma \ref{lemma:median_of_mean} is two fold. First scenario is the median of means for scalar random variable. Then for multivariate case, we have the robust mean estimation procedure from \cite{hsu2016loss}. Here we present the proof for both. \\

\textbf{Case 1:} \textit{(median of means)}\\
Since $\bar \mu_i = \frac{1}{m} \sum_{i=1}^{m} X_i$ with $X_i$ having mean $\mu$ and finite variance $\sigma^2$, then $\E[\bar \mu_i] = \mu$, $\bar \sigma^2= \frac{\sigma^2}{m}$. \\
By Chebyshev's inequality, $\P[|\bar \mu_i - \mu| \leq k_1 \frac{\sigma}{m}] \geq 1 - \frac{1}{k_1^2}$. Let $k_1 = 2\sqrt{n/k}$. Then we have,
\begin{align}
    \P\Big[|\mu_i - \mu| \leq \sqrt{\frac{4\sigma^2 k }{n}} \Big] \geq 1 - \frac{k}{4n} \geq \frac{3}{4}
\end{align}
Then apply Proposition 5 in \cite{hsu2016loss}, we have with $k= 8 \log(1/\delta)$, 
\begin{align}
    \P\Big[|\bar \mu - \mu|^2 \leq \frac{32\sigma^2 \log(1/\delta) }{n} \Big] \geq 1 - \delta
\end{align}


\textbf{Case 2:} \textit{(Robust mean estimator)}\\
Similar to the scalar case, we denote $\bar \bmu_i = \frac{1}{m} \sum_{i=1}^{m} X_i$ with $X_i$.
and $\E[\bar \bmu_i] = \bmu$. The variance is defined as the summation of variance of each element. $\bar \sigma^2= \sum_{i=1}^{d}\frac{\sigma_i^2}{m} = \frac{\sigma^2}{m}$, which is still finite.

By Chebyshev's inequality for finite dimensional vectors, we have,
\begin{align}
    &\P[\|\bar\bmu_i - \bmu \| \leq  k_2 \|\bar\sigma\| ] > 1 - \frac{1}{k_2^2}
\\\Rightarrow    
    &\P\Big[\|\bar\bmu_i - \bmu\| \leq \sqrt{\frac{3\sigma^2 k }{n}} \Big] \geq 1 - \frac{k}{3n} \geq \frac{2}{3}
\end{align}
where we plug in $k_2 = \sqrt{3n/k}$. 
If we denote $\varepsilon =  \sqrt{\frac{3\sigma^2 k }{n}}$, and apply Proposition 9 from \cite{hsu2016loss}, we will have,
\begin{align}
        \P\big[\|\bar \bmu_{i^*} - \bmu\| \leq 3\varepsilon \big] \geq 1 - e^{-\frac{k}{18}}
\end{align}
Then plug in $k = 18 \log(1/\delta)$,  $\varepsilon =  \sqrt{\frac{3\sigma^2 k }{n}}$ and denote $\bar\bmu = \bar \bmu_{i^*}$, we have,
\begin{align}
    \P\Big[\|\bar \bmu - \bmu\|^2 \leq 
    \frac{486\sigma^2 \log(1/\delta) }{n} \Big] \geq 1 - \delta   
\end{align}

That concludes the proof of lemma \ref{lemma:median_of_mean}.
\end{proof}

\section{Proofs in Section \ref{section:MSCG_non_robust} }

\subsection{Proof of Lemma \ref{lemma:MSCG_convergence} }

Before proving Lemma \ref{lemma:MSCG_convergence}, Theorem \ref{theorem:restart_MSCG} and Theorem \ref{theorem:robust_restart_MSCG}, we first provide two useful lemmas, which are proved in the subsequent subsections.
The following lemma shows that, under Assumption \ref{assumption:composite_problem}, $f(g(\w))$ has Lipschitz continuous gradients.

\begin{lem}\label{lemma:composite_smooth}
Suppose Assumption \ref{assumption:composite_problem} holds.
$F(\w)$ is $L$-smooth with $L = C_f L_g + C_g^2 L_g$.
\end{lem}

The following lemma proves the one-step result for Algorithm \ref{algorithm:MSCG}, which is further used to prove Lemma \ref{lemma:MSCG_convergence}, Theorem \ref{theorem:restart_MSCG} and Theorem \ref{theorem:robust_restart_MSCG}.

\begin{lem}
\label{lemma:one_step_recursion_MSCG}
Suppose Assumption \ref{assumption:composite_problem} holds. 
Let $\eta_t \leq 1/(2L)$.
For Algorithm \ref{algorithm:MSCG}, we have
\begin{align}\label{eq:one_step_recursion_MSCG}
\E[ F(\w_{t+1}) -  F(\w) ]
\leq &
\frac{1}{2\eta_t} \Big( \E[\| \w - \w_t \|^2] - \E[\| \w- \w_{t+1} \|^2] \Big) 
\nonumber\\
&
+ \frac{\mu}{4} \E[ \| \w - \w_t \|^2 ] 
+ 2\eta_t C_f^2 \frac{\sigma_1^2}{m_t}
+ 2\eta_t C_g^2 L_f^2 \frac{\sigma_0^2}{m_t}
+ \frac{C_g^2 L_f^2 \sigma_0^2}{\mu m_t}
\end{align}
\end{lem}

\begin{proof}

Let $\w = \w_t^*$, where $\w_t^*$ is the closest point in to $\w_t$ in the optimal set. 
Hence by the $\mu$-optimal strongly convexity of $F(\w)$, we have $F(\w_t) - F(\w_t^*) \geq \frac{\mu}{2} \| \w_t - \w_t^* \|^2$.

For (\ref{eq:one_step_recursion_MSCG}) of Lemma \ref{lemma:one_step_recursion_MSCG}, we let $\eta_t = \eta \leq 1/(2L)$, plug in $\w = \w_t^*$ and summation over over $t = 0, ..., T-1$ as follows
\begin{align}
\label{eq:sum_all_one_step_MSCG}
\sum_{t=0}^{T-1} \E[ F(\w_{t+1}) - F(\w_t^*) ]
\leq &
\frac{1}{2\eta} \sum_{t=0}^{T-1} \Big( \E [\| \w_t^* - \w_t \|^2] - \E [\| \w_{t}^* - \w_{t+1} \|^2] \Big) 
+ \sum_{t=0}^{T-1} \frac{\mu}{4} \E [ \|\w_t^* - \w_t\|^2 ]
\nonumber\\ 
&
+ \Big( 2\eta C_f^2 \sigma_1^2 
+ 2\eta C_g^2 L_f^2 \sigma_0^2 
+ \frac{ C_g^2 L_f^2 \sigma_0^2 }{ \mu } \Big)
\sum_{t=0}^{T-1} \frac{1}{ m_t }
\nonumber\\
\stackrel{(a)}\leq &
\frac{1}{2\eta} \sum_{t=0}^{T-1} \Big( \E [\| \w_t^* - \w_t \|^2] - \E [\| \w_{t+1}^* - \w_{t+1} \|^2] \Big) 
+ \sum_{t=0}^{T-1} \frac{1}{2} \E [ F(\w_t) - F(\w_t^*) ]
\nonumber\\
&
+ \Big( 2\eta C_f^2 \sigma_1^2 
+ 2\eta C_g^2 L_f^2 \sigma_0^2 
+ \frac{ C_g^2 L_f^2 \sigma_0^2 }{ \mu } \Big)
\sum_{t=0}^{T-1} \frac{1}{ m_t }
\nonumber\\
= & 
\frac{1}{2 \eta} \E [\| \w_0^* - \w_0 \|^2] - \frac{1}{2 \eta} \E [\| \w_{T}^* - \w_{T} \|^2] 
+ \sum_{t=0}^{T-1} \frac{1}{2} \E [ F(\w_t) - F(\w_t^*) ]
\nonumber\\
&
+ \Big( 2\eta C_f^2 \sigma_1^2 
+ 2\eta C_g^2 L_f^2 \sigma_0^2 
+ \frac{ C_g^2 L_f^2 \sigma_0^2 }{ \mu } \Big)
\sum_{t=0}^{T-1} \frac{1}{ m_t }
\nonumber\\
\leq &
\frac{1}{2 \eta} \| \w_0^* - \w_0 \|^2 
+ \sum_{t=0}^{T-1} \frac{1}{2} \E[ F(\w_t) - F(\w_t^*) ]
\nonumber\\
&
+ \Big( 2\eta C_f^2 \sigma_1^2 
+ 2\eta C_g^2 L_f^2 \sigma_0^2 
+ \frac{ C_g^2 L_f^2 \sigma_0^2 }{ \mu } \Big)
\sum_{t=0}^{T-1} \frac{1}{ m_t }
\nonumber\\
\end{align}
where (a) is due to the definition of $\w_t^*$ which further implies $\| \w_t^* - \w_{t+1} \|^2 \geq \| \w_{t+1}^* - \w_{t+1}\|^2$.


Re-arranging the above inequality, we have
\begin{align}
    \label{eq:sum_all_one_step_MSCG_combined}
\sum_{t=0}^{T-1} \E[ F(\w_{t+1}) - F(\w_t^*) ]
\leq 
&
\frac{1}{\eta} \| \w_0^* - \w_0 \|^2 
+ F(\w_0) - F(\w_T)
\nonumber\\
&
+ 2 \Big( 2\eta C_f^2 \sigma_1^2 
+ 2\eta C_g^2 L_f^2 \sigma_0^2 
+ \frac{ C_g^2 L_f^2 \sigma_0^2 }{ \mu } \Big)
\sum_{t=0}^{T-1} \frac{1}{ m_t }
\end{align}

Since $\w_t^*$ is an optimal solution, we have $F(\w_t^*) = F(\w_{t'}^*)$ for any $t$ and $t'$.
By applying Jensen's inequality and the convexity of $F$ to LHS of the above (\ref{eq:sum_all_one_step_MSCG_combined}), we have
\begin{align}
\label{eq:converge_MSCG_large_condition_number_2}
\E[ F(\wh_T) - \min_w F(\w) ]
\leq &
\frac{ \| \w_0^* - \w_0 \|^2 }{\eta T}
+ \frac{F(\w_0) - F(\w_T)}{T}
\nonumber\\
&
+ \frac{2}{T}\Big( 2\eta C_f^2 \sigma_1^2 
+ 2\eta C_g^2 L_f^2 \sigma_0^2 
+ \frac{ C_g^2 L_f^2 \sigma_0^2 }{ \mu } \Big)
\sum_{t=0}^{T-1} \frac{1}{ m_t } 
\end{align}
where $\wh_T = \frac{1}{T} \sum_{t=1}^T \w_t$.
\\
Further by the assumption that $F(\w_0) - F(\w_T) \leq \epsilon_0$, we obtain Lemma \ref{lemma:MSCG_convergence}.
\begin{align}
\label{eq:Lemma2_overall}
\E[ F(\wh_T) - \min_w F(\w) ]
\leq &
\frac{\epsilon_0}{T}
+ \frac{1}{T}\sum_{t=0}^{T-1} \frac{4\eta C_f^2 \sigma_1^2 + 4\eta C_g^2 L_f^2 \sigma_0^2 + 2 C_g^2 L_f^2 \sigma_0^2 /\mu }{m_t})
+ \frac{ \| \w_0^* - \w_0 \|^2 }{\eta T}
\nonumber\\
\end{align}

If we let $m_t = \frac{t+1}{\mu}$ for $t=0, ..., T-1$, then since $\sum_{t=1}^{T} \frac{1}{t} \leq 1 + \int_{2}^{T} \frac{1}{t} dt \leq \ln T + 1$, we can further have 
\begin{align}
\label{eq:Lemma2_with_mt}
\E[ F(\wh_T) - \min_w F(\w) ]
\leq &
\frac{\epsilon_0}{T}
+  \big(4\eta C_f^2 \sigma_1^2 + 4\eta C_g^2 L_f^2 \sigma_0^2 + 2 C_g^2 L_f^2 \sigma_0^2 /\mu \big)
\frac{\ln T + 1}{T}
+ \frac{ \| \w_0^* - \w_0 \|^2 }{\eta T}
\nonumber\\
\end{align}

To guarantee $\E[F(\wh_T) - F(\w_*)] \leq \epsilon$, we require $T = \tilde O(1/\epsilon)$ and the total sample complexity is 
$$
\sum_{t=0}^{T-1} m_t 
= \sum_{t=0}^{T-1} \frac{t+1}{\mu}
= \frac{ (T+1) T }{ 2\mu }
= \tilde O(\frac{1}{\mu \epsilon^2})   .
$$
\end{proof}

\subsection{Proof for Theorem \ref{theorem:restart_MSCG} }
\begin{proof}

We first consider the convergence of the inner loop.
Suppose $F(\w_0) - F(\w_*) \leq \epsilon_0$,
$m_t = m = \frac{1}{ c \mu \epsilon_0 }$ where
$c = \frac{1}{ 4 ( 2 \mu \eta C_f^2 \sigma_1^2 
+ 2\mu \eta C_g^2 L_f^2 \sigma_0^2 
+  C_g^2 L_f^2 \sigma_0^2 ) }$.
Starting from (\ref{eq:sum_all_one_step_MSCG}) where we take summation over the one-step result, let $\w = \w_*$ and we have
\begin{align*}
\sum_{t=0}^{T-1} \E[ F(\w_{t+1}) - F(\w_*) ]
\leq &
\frac{1}{2 \eta} \| \w_* - \w_0 \|^2 
+ \Big( 2\eta C_f^2 \sigma_1^2 
+ 2\eta C_g^2 L_f^2 \sigma_0^2 
+ \frac{ C_g^2 L_f^2 \sigma_0^2 }{ \mu } \Big)
\sum_{t=0}^{T-1} \frac{1}{ m }
\\
= &
\frac{1}{2 \eta} \| \w_* - \w_0 \|^2 
+ \Big( 2 \eta C_f^2 \sigma_1^2 
+ 2 \eta C_g^2 L_f^2 \sigma_0^2 
+ \frac{ C_g^2 L_f^2 \sigma_0^2 }{ \mu } \Big)
\sum_{t=0}^{T-1} c \mu \epsilon_0
\\
\leq &
\frac{ F(\w_0) - F(\w_*) }{\mu\eta}
+ \Big( 2 \eta C_f^2 \sigma_1^2 
+ 2 \eta C_g^2 L_f^2 \sigma_0^2 
+ \frac{ C_g^2 L_f^2 \sigma_0^2 }{ \mu } \Big) T c \mu \epsilon_0   ,
\end{align*}
where the last inequality is due to $\mu$-optimal strong convexity of $F$.

Let $\epsilon_0 = F(\w_0) - F(\w_*)$.
Applying Jensen's inequality to LHS of the above inequality, we have
\begin{align*}
\E[ F(\wh_T) - F(\w) ]
\leq &
\frac{ \epsilon_0 }{\mu \eta T}
+ \Big( 2 \mu \eta C_f^2 \sigma_1^2 
+ 2 \mu \eta C_g^2 L_f^2 \sigma_0^2 
+ C_g^2 L_f^2 \sigma_0^2 \Big) c \epsilon_0
\\
= &
\frac{\epsilon_0}{4} 
+ \frac{\epsilon_0}{4}
= \frac{ \epsilon_0 }{2}   ,
\end{align*}
where the last inequality is due to
\begin{align*}
T = \frac{4}{\mu \eta}
\text{ and }
c = \frac{1}{ 4 ( 2 \mu \eta C_f^2 \sigma_1^2 
+ 2 \mu \eta C_g^2 L_f^2 \sigma_0^2 
+ C_g^2 L_f^2 \sigma_0^2 ) }  .
\end{align*}

Then we consider the two consecutive loops.
Given $F(\w_{k-1}) - F(\w_*) \leq \epsilon_{k-1}$, we have $F(\w_k) - F(\w_*) \leq \frac{ \epsilon_{k-1} }{2} = \epsilon_{k}$, as long as we set $m_k = \frac{1}{c\mu \epsilon_{k-1}}$.
To achieve an $\epsilon$-optimal solution, i.e., $\epsilon_K \leq \epsilon$, we require $K = \lceil \log(\frac{\epsilon_0}{\epsilon}) \rceil$,
which leads to the total sample complexity
$$
m_{tot} 
= \sum_{k=1}^K m_k T
= \sum_{k=1}^K \frac{ T }{c \mu} \cdot \frac{1}{ \epsilon_{k-1} }
= \frac{T}{c \mu} \sum_{k=1}^K \frac{ 2^{k-1} }{\epsilon_0}
= \frac{T}{ c \mu \epsilon_0 } (2^K - 1)
\leq \frac{4}{ c \eta \mu^2 \epsilon }   .
$$
\end{proof}

\subsection{Proof of Lemma \ref{lemma:composite_smooth}}\label{subsection:proof:lemma:composite_smooth}

\begin{proof}
We prove $f(g(\w))$ has Lipschitz continuous gradients as follows,
\begin{align*}
&
\| \nabla f(g(\w)) \nabla g(\w) - \nabla f(g(\w')) \nabla g(\w') \|
\\
\stackrel{(a)}{\leq} &
\| \nabla f(g(\w)) \nabla g(\w) - \nabla f(g(\w)) \nabla g(\w') \|
+ \| \nabla f(g(\w)) \nabla g(\w') - \nabla f(g(\w')) \nabla g(\w') \|
\\
= &
\| \nabla f(g(\w)) \| \cdot \| \nabla g(\w) - \nabla g(
\w') \|
+ \| \nabla g(\w') \| \cdot \| \nabla f(g(\w)) - \nabla f(g(\w')) \|
\\
\stackrel{(b)}{\leq} &
C_f L_g \| \w - \w' \|
+ C_g L_g \| g(\w) - g(\w') \|
\\
\stackrel{(c)}{\leq} &
(C_f L_g + C_g^2 L_g) \| \w - \w' \|  ,
\end{align*}
where inequality $(a)$ is due to triangle inequality,
inequality $(b)$ is due to $C_f$-Lipschitz continuity of $f$, $L_g$-smoothness of $g$ and $C_g$-Lipschitz continuity of $g$.
Inequality $(c)$ is due to $C_g$-Lipschitz continuity of $g$.
\end{proof}

\subsection{Proof of Lemma \ref{lemma:one_step_recursion_MSCG} }\label{subsection:proof:lemma:one_step_recursion_MSCG}

\begin{proof}
By the optimality of $\w_{t+1}$ of Line \ref{algorithm:MSCG:line:update} of Algorithm \ref{algorithm:MSCG}, we have
\begin{align}\label{eq:optimality_one_step}
&
\langle \nabla f(\y_{t+1}) \z_{t+1} + \partial r(\w_{t+1}) + \frac{1}{\eta_t} (\w_{t+1} - \w_t), \w - \w_{t+1} \rangle \geq 0
\nonumber\\
\Rightarrow & ~~
\langle \nabla f(\y_{t+1}) \z_{t+1} + \partial r(\w_{t+1}), \w_{t+1} - \w \rangle
\leq
\langle \frac{1}{\eta_t} (\w_{t+1} - \w_t), \w - \w_{t+1} \rangle
\nonumber\\
& ~~ 
=
\frac{1}{\eta_t} \langle \w_{t+1} - \w + \w - \w_t, \w - \w_{t+1} \rangle
\nonumber\\
& ~~ 
=
\frac{1}{\eta_t} \Big( -\| \w_{t+1} - \w \| + \frac{1}{2} \| \w - \w_t \|^2 + \frac{1}{2} \| \w - \w_{t+1} \|^2 - \frac{1}{2} \| \w - \w_t - \w + \w_{t+1} \|^2 \Big)
\nonumber\\
& ~~
=
\frac{1}{2\eta_t} \Big( \| \w - \w_t \|^2 - \| \w - \w_{t+1} \|^2 - \| \w_t - \w_{t+1} \|^2 \Big)      .
\end{align}

For the LHS of the above inequality (\ref{eq:optimality_one_step}), we further have the following lower bound
\begin{align}
\label{eq:lhs_of_optimality_one_step}
&
\langle \nabla f(\y_{t+1}) \z_{t+1} + \partial r(\w_{t+1}), \w_{t+1} - \w \rangle
\nonumber\\
= &
\langle \nabla f(g(\w_t))\nabla g(\w_t) + \partial r(\w_{t+1}), \w_{t+1} - \w \rangle
\nonumber\\
&
+ \langle \nabla f(\y_{t+1}) \z_{t+1} - \nabla f(g(\w_t))\nabla g(\w_t), \w_{t+1} - \w \rangle
\nonumber\\
= &
\langle \nabla f(g(\w_t))\nabla g(\w_t), \w_{t+1} - \w_t + \w_t - \w \rangle
+ \langle \partial r(\w_{t+1}), \w_{t+1} - \w \rangle
\nonumber\\
&
+ \langle \nabla f(\y_{t+1}) \z_{t+1} - \nabla f(g(\w_t))\nabla g(\w_t), \w_{t+1} - \w \rangle
\nonumber\\
\stackrel{(a)}{\geq} &
\Big( f(g(\w_{t+1})) - f(g(\w_t)) \Big)
- \frac{L}{2} \| \w_{t+1} - \w_t \|^2
- \Big( f(g(\w_t)) - f(g(\w)) \Big) 
\nonumber\\
&
+ r(\w_{t+1}) - r(\w)
+ \langle \nabla f(\y_{t+1}) \z_{t+1} - \nabla f(g(\w_t))\nabla g(\w_t), \w_{t+1} - \w \rangle
\nonumber\\
= &
F(\w_{t+1}) - F(\w)
- \frac{L}{2} \| \w_{t+1} - \w_t \|^2
\nonumber\\
&
+ \langle \nabla f(\y_{t+1}) \z_{t+1} - \nabla f(g(\w_t))\nabla g(\w_t), \w_{t+1} - \w \rangle
\end{align}
where inequality $(a)$ is due to convexity of $r$ as well as $L$-smoothness and convexity of $f \circ g$.

Combining the above two inequalities (\ref{eq:optimality_one_step}) and (\ref{eq:lhs_of_optimality_one_step}), we have
\begin{align}
\label{eq:optimality_one_step_before_bound_inner_product}
F(\w_{t+1}) - F(\w)
\leq &
\frac{1}{2\eta_t} \Big( \| \w - \w_t \|^2 - \| \w - \w_{t+1} \|^2 - \| \w_t - \w_{t+1} \|^2 \Big)
\nonumber\\
&
+ \frac{L}{2} \| \w_t - \w_{t+1} \|^2
\nonumber\\
&
+ \underbrace{ \langle \nabla f(g(\w_t))\nabla g(\w_t) - \nabla f(\y_{t+1}) \z_{t+1} , \w_{t+1} - \w \rangle }_{:= A}   .
\end{align}

To upper bound term $A$ in (\ref{eq:optimality_one_step_before_bound_inner_product}), we have 
\begin{align}
\label{eq:one_step_upper_bound_A}
A = &
\langle \nabla f(g(\w_t))\nabla g(\w_t) - \nabla f(\y_{t+1}) \z_{t+1} , \w_{t+1} - \w_t + \w_t - \w \rangle
\nonumber\\
\leq &
\eta_t \| \nabla f(g(\w_t)) \nabla g(\w_t) - \nabla f(\y_{t+1}) \z_{t+1} \|^2
+ \frac{\| \w_t - \w_{t+1} \|^2}{4\eta_t}
\nonumber\\
&
+ \langle \nabla f(g(\w_t))\nabla g(\w_t) - \nabla f(\y_{t+1}) \z_{t+1} , \w_t - \w \rangle  ,
\end{align}
where the last inequality is due to Young's inequality.

Plugging (\ref{eq:one_step_upper_bound_A}) into (\ref{eq:optimality_one_step_before_bound_inner_product}), we finally have
\begin{align}\label{eq:optimality_one_step_after_bound_inner_product}
F(\w_{t+1}) - F(\w)
\leq &
\frac{1}{2\eta_t} \Big( \| \w - \w_t \|^2 - \| \w - \w_{t+1} \|^2 \Big)
\nonumber\\
&
+ \Big( \frac{L}{2} - \frac{1}{4 \eta_t} \Big) \| \w_t - \w_{t+1} \|^2
\nonumber\\
&
+ \eta_t \| \nabla f(g(\w_t)) \nabla g(\w_t) - \nabla f(\y_{t+1}) \z_{t+1} \|^2
\nonumber\\
&
+ \langle \nabla f(g(\w_t))\nabla g(\w_t) - \nabla f(\y_{t+1}) \z_{t+1} , \w_t - \w \rangle
\nonumber\\
\leq &
\frac{1}{2\eta_t} \Big( \| \w - \w_t \|^2 - \| \w - \w_{t+1} \|^2 \Big)
\nonumber\\
&
+ \eta_t \| \nabla f(g(\w_t)) \nabla g(\w_t) - \nabla f(\y_{t+1}) \z_{t+1} \|^2
\nonumber\\
&
+ \langle \nabla f(g(\w_t))\nabla g(\w_t) - \nabla f(\y_{t+1}) \z_{t+1} , \w_t - \w \rangle   ,
\end{align}
where the last inequality is due to the assumption $\eta_t \leq \frac{1}{2L}$.

Recall the following conditions
\begin{align}\label{eq:bounds}
& \E[\z_{t+1}] = \nabla g(\w_t),
&& \E[\| \z_{t+1} \|^2] \leq C_g^2,
\nonumber\\
& \E[\| \z_{t+1} \|] \leq \sqrt{\E[\| \z_{t+1} \|^2]} \leq C_g,
&& \E[ \nabla f(\cdot) ] \leq C_f .
\nonumber\\
& \E[ \y_{t+1} - g(\w_t) \|^2] \leq \frac{\sigma_0^2}{m_t},
&& \E[\| \z_{t+1} - \nabla g(\w_t) \|^2] \leq \frac{\sigma_1^2}{m_t},
\end{align}

By taking expectation conditioned on $\w_t$ on both sides of (\ref{eq:optimality_one_step_after_bound_inner_product}), we have
\begin{align}\label{eq:MSCG_one_step1}
&
\E[ F(\w_{t+1}) - F(\w)]
\nonumber\\
\leq &
\frac{1}{2\eta_t} \Big( \E[ \| \w - \w_t \|^2 ] - \E[ \| \w - \w_{t+1} \|^2 ] \Big)
\nonumber\\
&
+\eta_t \E[ \| \nabla f(g(\w_t)) \nabla g(\w_t) - \nabla f(\y_{t+1}) \z_{t+1} \|^2 ]
\nonumber\\
&
+ \E[ \langle \nabla f(g(\w_t))\nabla g(\w_t) - \nabla f(\y_{t+1}) \z_{t+1} , \w_t - \w \rangle ]
\nonumber\\
\leq &
\frac{1}{2\eta_t} \Big( \E[ \| \w - \w_t \|^2 ] - \E[ \| \w - \w_{t+1} \|^2 ] \Big)
\nonumber\\
&
+ \underbrace{ 2\eta_t \E[ \| \nabla f(g(\w_t)) \nabla g(\w_t) - \nabla f(g(\w_t)) \z_{t+1} \|^2 ] }_{ := A }
\nonumber\\
&
+ \underbrace{ \E[ \langle \nabla f(g(\w_t))\nabla g(\w_t) - \nabla f(g(\w_t)) \z_{t+1} , \w_t - \w \rangle ] }_{ := B }
\nonumber\\
&
+ \underbrace{ 2\eta_t \E[ \| \nabla f(g(\w_t)) \z_{t+1} - \nabla f(\y_{t+1}) \z_{t+1} \|^2 ] }_{ := C }
\nonumber\\
&
+ \underbrace{ \E[ \langle \nabla f(g(\w_t)) \z_{t+1} - \nabla f(\y_{t+1}) \z_{t+1} , \w_t - \w \rangle ] }_{ := D }   .
\end{align}

Then we upper bound the four terms above $A, B, C, D$ by conditions in (\ref{eq:bounds}) as follows.
\begin{align}
\label{eq:MSCG_bound_A}
A \leq
2\eta_t \E[ \| \nabla f(g(\w_t)) \|^2 \cdot \| \nabla g(\w_t) - \z_{t+1} \|^2 ]
\leq 
2\eta_t C_f^2 \frac{\sigma_1^2}{m_t}  .
\end{align}

\begin{align}
\label{eq:MSCG_bound_B}
B 
=
\langle \nabla f(g(\w_t)) \nabla g(\w_t) - \nabla f(g(\w_t)) \E [ \z_{t+1} ] , \w_t - \w \rangle
= 0  ,
\end{align}
where the first equality is due to the fact that $\w_t$ is independent on the randomness, and the second inequality is due to $\E[\z_{t+1}] = \nabla g(\w_t)$ in (\ref{eq:bounds}).

\begin{align}
\label{eq:MSCG_bound_D}
C 
\leq
2\eta_t \E [ \| \z_{t+1} \|^2 \cdot \| \nabla f(g(\w_t)) - \nabla f(\y_{t+1}) \|^2 ]
\leq 
2\eta_t C_g^2 L_f^2 \frac{\sigma_0^2}{m_t}  .
\end{align}

\begin{align}
\label{eq:MSCG_bound_E}
D
\leq &
\E [ \| \z_{t+1} \| \cdot \| \nabla f(g(\w_t)) - \nabla f(\y_{t+1}) \| \cdot \| \w - \w_t \| ]
\nonumber\\
\leq &
\E [ \| \z_{t+1} \| \cdot L_f \| g(\w_t) - \y_{t+1} \| \cdot \| \w - \w_t \| ]
\nonumber\\
\leq &
\E [ \| \z_{t+1} \| ] \cdot \E[ L_f \| g(\w_t) - \y_{t+1} \| \cdot \| \w - \w_t \| ]
\nonumber\\
\leq &
\E[ C_g L_f \| g(\w_t) - \y_{t+1} \| \cdot \| \w - \w_t \| ]
\nonumber\\
\leq &
\E \Big[ \frac{C_g^2 L_f^2}{\mu} \| g(\w_t) - \y_{t+1} \|^2 + \frac{\mu \| \w - \w_t \|^2}{4} \Big]
\nonumber\\
\leq &
\frac{C_g^2 L_f^2 \sigma_0^2}{\mu m_t}
+ \frac{\mu}{4} \E[ \| \w - \w_t \|^2 ]   ,
\end{align}
where the first inequality is due to Cauchy-Schwarz inequality, 
the second inequality is due to $L_f$-smooth of $f$,
the third inequality is due to independence of $\calS_1$ and $\calS_2$,
the fourth inequality is due to conditions in (\ref{eq:bounds}),
the fifth inequality is due to Young's inequality,
and the last inequality is due to conditions in (\ref{eq:bounds}).

Plugging (\ref{eq:MSCG_bound_A}), (\ref{eq:MSCG_bound_B}), (\ref{eq:MSCG_bound_D}) and (\ref{eq:MSCG_bound_E}) into (\ref{eq:MSCG_one_step1}), we have
\begin{align}
\label{eq:MSCG_one_step2}
&
\E[ F(\w_{t+1}) - F(\w) ]
\nonumber\\
\leq &
\frac{1}{2\eta_t} \Big( \E[\| \w - \w_t \|^2] - \E[\| \w - \w_{t+1} \|^2] \Big) 
+ \frac{\mu}{4} \E[ \| \w - \w_t \|^2 ]
\nonumber\\
&
+ 2\eta_t C_f^2 \frac{\sigma_1^2}{m_t}
+ 2\eta_t C_g^2 L_f^2 \frac{\sigma_0^2}{m_t}
+ \frac{C_g^2 L_f^2 \sigma_0^2}{\mu m_t}
\nonumber\\
= &
\frac{1}{2\eta_t} \Big( \E[\| \w - \w_t \|^2] - \E[\| \w - \w_{t+1} \|^2] \Big) 
+ \frac{\mu}{4} \E[ \| \w - \w_t \|^2 ] 
\nonumber\\
&
+ \Big(2\eta_t C_f^2 \sigma_1^2 
+ 2\eta_t C_g^2 L_f^2 \sigma_0^2
+ \frac{C_g^2 L_f^2 \sigma_0^2}{\mu} \Big) \frac{1}{m_t}
\end{align}

\end{proof}  

\subsection{Proof of Lemma \ref{lemma:RMSCG_one_stage_0} }

\begin{proof} (of Lemma \ref{lemma:RMSCG_one_stage_0})
Start with the one-step update of (\ref{eq:optimality_one_step_after_bound_inner_product}) in proof of Lemma \ref{lemma:one_step_recursion_MSCG}.
We let $\eta_t = \eta \leq 1/(2L)$ and $\w = \w^*_t$, i.e., the closted optimal solution to $\w_t$.
Then taking summation over $t=0, ..., T-1$, we have
\begin{align*}
&
\sum_{t=0}^{T-1} F(\w_{t+1}) - F(\w_t^*)
\leq 
\frac{1}{2\eta} \sum_{t=0}^{T-1} \Big( \| \w_t^* - \w_t \|^2 - \| \w_t^* - \w_{t+1} \|^2 \Big)
\nonumber\\
& + \eta \| \nabla f(g(\w_t)) \nabla g(\w_t) - \nabla f( \y_{t+1})  \z_{t+1} \|^2
+ \langle \nabla f(g(\w_t))\nabla g(\w_t) - \nabla f( \y_{t+1})  \z_{t+1} , \w_t - \w_t^* \rangle .
\end{align*}

For LHS, since the optimal values are equal, i.e., $F(\w_t^*) = F(\w_{t'}^*)$, we apply Jensen's inequality as follows
\begin{align*}
&
\frac{1}{T} \sum_{t=0}^{T-1} ( F(\w_{t+1}) - F(\w^*_t) )
\leq 
\frac{1}{2\eta T} \sum_{t=0}^{T-1} \Big( \| \w_t^* - \w_t \|^2 - \| \w_t^* - \w_{t+1} \|^2 \Big)
\nonumber\\
& 
+ \frac{\eta}{T} \sum_{t=0}^{T-1} \| \nabla f(g(\w_t)) \nabla g(\w_t) - \nabla f( \y_{t+1})  \z_{t+1} \|^2
+ \frac{1}{T} \sum_{t=0}^{T-1} \langle \nabla f(g(\w_t))\nabla g(\w_t) - \nabla f( \y_{t+1})  \z_{t+1} , \w_t - \w_t^* \rangle
\\
&
\leq
\frac{1}{2\eta T} \| \w_0^* - \w_0 \|^2 
+ \frac{\eta}{T} \sum_{t=0}^{T-1} \| \nabla f(g(\w_t)) \nabla g(\w_t) - \nabla f( \y_{t+1})  \z_{t+1} \|^2
\nonumber\\
& ~~~~~ 
+ \frac{1}{T} \sum_{t=0}^{T-1} \langle \nabla f(g(\w_t))\nabla g(\w_t) - \nabla f( \y_{t+1})  \z_{t+1} , \w_t - \w_t^* \rangle ,
\end{align*}
where the last inequality is due to the definition of $\w_t^*$: $\| \w_{t+1}^* - \w_{t+1} \| \leq \| \w_t^* - \w_{t+1} \|$.
\end{proof}

\subsection{Proof of Theorem \ref{theorem:robust_restart_MSCG} }
\begin{proof}

As analyzed, the key is to bound the two terms $A$ and $B$ in (\ref{eq:RMSCG_one_stage_0}) of Lemma \ref{lemma:RMSCG_one_stage_0}.
\begin{align*}
&
\frac{1}{T} \sum_{t=0}^{T-1} ( F(\w_{t+1}) - F(\w^*_t) )
\leq 
\frac{1}{2\eta T} \| \w_0^* - \w_0 \|^2 
+ \frac{\eta}{T} \underbrace{ \sum_{t=0}^{T-1} \| \nabla f(g(\w_t)) \nabla g(\w_t) - \nabla f( \y_{t+1})  \z_{t+1} \|^2 }_{ A }
\nonumber\\
& ~~~~~ 
+ \frac{1}{T} \underbrace{ \sum_{t=0}^{T-1} \langle \nabla f(g(\w_t))\nabla g(\w_t) - \nabla f( \y_{t+1})  \z_{t+1} , \w_t - \w_t^* \rangle }_{ B } ,
\end{align*}
which can be further decomposed to
\begin{align*}
A
\leq &
2 \underbrace{\sum_{t=0}^{T-1} \|\nabla g(\w_t) ^{\top}\nabla f(g(\w_t))-  \nabla g(\w_t)^{\top}\nabla f(\y_{t+1})\|^2}_{A_1} 
\\
& + 2   \underbrace{ \sum_{t=0}^{T-1}  \|\nabla g(\w_t)^{\top} \nabla f(\y_{t+1})  - \z_{t+1}^{\top} \nabla f(\y_{t+1}) \|^2}_{A_2}
\nonumber\\
B 
\leq &
\underbrace{ \sum_{t=0}^{T-1} \langle \nabla g(\w_t)^\top \nabla f(g(\w_t)) - \z_{t+1}^\top \nabla f(g(\w_t)) , \w_t - \w_t^* \rangle}_{B_1}
\\
&
+ \underbrace{ \sum_{t=0}^{T-1} \langle \z_{t+1}^\top \nabla f(g(\w_t)) - \z_{t+1}^\top \nabla f(\y_{t+1}) , \w_t - \w_t^* \rangle}_{B_2}
\end{align*}

Under Assumption \ref{assumption:composite_problem}, these four terms have the following upper bounds:
\begin{align*}
A_1
\leq &
\sum_{t=0}^{T-1} \| \z_{t+1} \|^2 \cdot \| \nabla f(g(\w_t)) - \nabla f(\y_{t+1}) \|^2 
\leq 
C_g^2 L_f^2 \sum_{t=0}^{T-1} \|  g(\w_t) - \y_{t+1} \|^2    ,
\\
A_2 
\leq &
 \sum_{t=0}^{T-1} \| \nabla f(g(\w_t)) \|^2 \cdot \| \nabla g(\w_t) - \z_{t+1} \|^2 
\leq 
C_f^2 \sum_{t=0}^{T-1} \| \nabla g(\w_t) - \z_{t+1} \|^2    ,
\\
B_1
\leq & 
\sum_{t=0}^{T-1} \| \nabla f(g(\w_t))\| \cdot \| \nabla g(\w_t) - \z_{t+1} \| \cdot \| \w_t^* - \w_t \| 
\nonumber\\
\leq &
\sum_{t=0}^{T-1} C_f \| \nabla g(\w_t) - \z_{t+1} \| \cdot \| \w_t^* - \w_t \| 
\stackrel{(a)}{\leq} 
\frac{2 C_f^2 }{\mu} \sum_{t=0}^{T-1} \| \nabla g(\w_t) - \z_{t+1} \|^2 + \frac{\mu }{8} \sum_{t=0}^{T-1} \| \w_t^* - \w_t \|^2
\\
B_2 
\leq &  
\sum_{t=0}^{T-1} \| \z_{t+1} \| \cdot \| \nabla f(g(\w_t)) - \nabla f(\y_{t+1}) \| \cdot \| \w_t^* - \w_t \| 
\nonumber\\
\leq &
\sum_{t=0}^{T-1} C_g L_f \| g(\w_t) - \y_{t+1} \| \cdot \| \w_t^* - \w_t \| 
\nonumber\\
\stackrel{(b)}{\leq} &
\frac{2 C_g^2 L_f^2}{\mu} \sum_{t=0}^{T-1} \| g(\w_t) - \y_{t+1} \|^2 + \frac{\mu}{8} \sum_{t=0}^{T-1} \| \w_t^* - \w_t \|^2   ,
\end{align*}
where inequalities $(a)$ and $(b)$ are due to Young's inequality.

By Lemma \ref{lemma:median_of_mean}, the following inequalities hold with probability at leaset $1-\delta$, respectively.
\begin{align*}
\|g(\w_t) - \y_{t+1}\|^2 \leq \frac{486\sigma_0^2 \log(1/\delta)}{m_t}   ,~~~~
\|\nabla g(\w_t) - \z_{t+1} \|^2 \leq \frac{486\sigma_1^2 \log(1/\delta)}{m_t}   .
\end{align*}

Combining the above upper bounds for $A$ and $B$ together, from (\ref{eq:RMSCG_one_stage_0}) we have 
\begin{align*}
&
\frac{1}{T} \sum_{t=0}^{T-1} ( F(\w_{t+1}) - F(\w_*) )
\leq 
\frac{ \| \w_0^* - \w_0 \|^2 }{2 \eta T}
+ \frac{\mu}{4 T} \sum_{t=0}^{T-1} \| \w_t^* - \w_t \|^2
\\
&
+ \frac{1}{T} ( 2 \eta C_g^2 L_f^2 + \frac{ 2 C_g^2 L_f^2 }{ \mu } ) \sum_{t=0}^{T-1} \| g(\w_t) - \y_{t+1} \|^2
+ \frac{1}{T} ( 2 \eta C_f^2 + \frac{2 C_f^2}{\mu} ) \sum_{t=0}^{T-1} \| \nabla g(\w_t) - \z_{t+1} \|^2
\\
&
\leq 
\frac{ F(\w_0) - F(\w_0^*) }{ \mu \eta T }
+ \frac{1}{2 T} \sum_{t=0}^{T-1} ( F(\w_t) - F(\w_t^*) )
\\
& 
+ 2 C_g^2 L_f^2 ( \eta + \frac{ 1 }{ \mu } ) 486 \sigma_0^2 \log(1/\delta) \frac{1}{T} \sum_{t=0}^{T-1} \frac{1}{m_t}
+ 2 C_f^2 ( \eta + \frac{1}{\mu} ) 486 \sigma_1^2 \log(1/\delta) \frac{1}{T} \sum_{t=0}^{T-1} \frac{1}{m_t} .
\end{align*}

Re-arranging the above inequality, we have
\begin{align*}
&
\frac{1}{2 T} \sum_{t=0}^{T-1} ( F(\w_{t+1}) - F(\w_{t+1}^*) )
\\
&
\leq 
\frac{ \epsilon_0 }{ \mu \eta T }
+ \frac{\epsilon_0}{2 T}
+ 2 ( \eta + \frac{1}{\mu} ) ( C_g^2 L_f^2 \sigma_0^2 + C_f^2 \sigma_1^2 ) 486 \log(1/\delta) \frac{1}{T} \sum_{t=0}^{T-1} \frac{1}{m_t}
\\
& =
\frac{ \epsilon_0 }{ \mu \eta T }
+ \frac{\epsilon_0}{2 T}
+ 2 ( \mu \eta + 1 ) ( C_g^2 L_f^2 \sigma_0^2 + C_f^2 \sigma_1^2 ) 486 \log(1/\delta) \frac{1}{\mu m} .
\end{align*}
For LHS, we apply Jensen's inequality to derive the lower bound $F(\wh_T) - F(\w_*)$.

To guarantee $F(\wh_T) - F(\w_*) \leq \frac{\epsilon_0}{2}$, we have
\begin{align*}
& 
\frac{\epsilon_0}{T} ( \frac{2}{\mu \eta} + 1 ) = \frac{\epsilon_0}{4}
\qquad \Leftarrow \qquad 
T = 4 ( \frac{2}{\mu \eta} + 1 ) 
\\
&
4 ( \mu \eta + 1 ) ( C_g^2 L_f^2 \sigma_0^2 + C_f^2 \sigma_1^2 ) 486 \log(1/\delta) \frac{1}{\mu m} 
= 
\frac{\epsilon_0}{4}
\\
& \qquad \qquad \qquad \qquad \qquad \qquad 
\Leftarrow
m = \frac{16}{\mu \epsilon_0} ( \mu \eta + 1 ) ( C_g^2 L_f^2 \sigma_0^2 + C_f^2 \sigma_1^2 ) 486 \log(1/\delta).
\end{align*}

As a result, for Algorithm \ref{algorithm:RMSCG}, given $F(\w_{k-1}) - F(\w_*) \leq \epsilon_{k-1}$ at each stage, we require 
\begin{align*}
T \geq 4 ( \frac{2}{\mu \eta} + 1 ) , ~~~~
m_k \geq \frac{16}{\mu \epsilon_{k-1}} ( \mu \eta + 1 ) ( C_g^2 L_f^2 \sigma_0^2 + C_f^2 \sigma_1^2 ) 486 \log(1/\delta),
\end{align*}
so that we can guarantee the following inequality holds with probability at least $1 - 2T \delta$
\begin{align*}
F(\w_k) - F(\w_*) \leq \frac{\epsilon_{k-1}}{2}   .
\end{align*}

To reach an $\epsilon$-optimal solution with probability at least $1 - 2TK\delta$, we also require $K = \lceil \log(\frac{\epsilon_0}{\epsilon}) \rceil$ and the total sample complexity is
\begin{align*}
m_{tot} = \sum_{k=1}^K m_k T
= \sum_{k=1}^K O( \frac{1}{\mu^2 \epsilon_{k-1}} ) 
= O( \frac{1}{\mu^2 \epsilon} )  .
\end{align*}

\end{proof}

\section{Proofs in Section \ref{section:RMSCG_ref_robust} }

\subsection{Proof of Theorem \ref{theorem:RROSC_ref_convergence} }\label{section:theorem:ref_convergence}

In this section, we prove Theorem \ref{theorem:RROSC_ref_convergence}.
As we mentioned, the key is to bound the two terms $A$ and $B$ in Lemma \ref{lemma:RMSCG_one_stage_0} by making use of the reference truncation technique in (\ref{eq:RMSCG_ref_truncation_y}) and (\ref{eq:RMSCG_ref_truncation_g}).

\begin{proof} (details of Theorem \ref{theorem:RROSC_ref_convergence})
Start from the one-stage result in Lemma \ref{lemma:RMSCG_one_stage_0}.
Recall that by the robust estimator, we have $\tilde \y_0$ and $\tilde \z_0$ with probability at least $1 - 2\delta$ such that 
\begin{align*}
\| \tilde \y_0 - g(\w_0) \| \leq \nu \sigma_0   , ~~~~~ \| \tilde \z_0 - \nabla g(\w_0) \| \leq \nu \sigma_1 ,
\end{align*}
which satisfy (\ref{eq:RMSCG_ref_truncation_y}) and (\ref{eq:RMSCG_ref_truncation_g}).

By (\ref{eq:RMSCG_one_stage_0}), replacing $\y_{t+1}$ and $\z_{t+1}$ with $\yh_{t+1}$ and $\zh_{t+1}$, we have
\begin{align}\label{eq:RMSCG_one_stage_robust_1}
F(\wh_T) - F(\w_*)
\leq &
\frac{ \| \w_0 - \w_* \|^2 }{ 2 \eta T }
+ \frac{ \eta }{ T } \sum_{t=0}^{T-1} \| \nabla g(\w_t)^\top \nabla f(g(\w_t))  - \zh_{t+1}^\top \nabla f(\yh_{t+1})  \|^2
\nonumber\\
&
+ \frac{1}{T} \sum_{t=0}^{T-1} \langle \nabla g(\w_t)^\top \nabla f(g(\w_t))  - \zh_{t+1}^\top \nabla f(\yh_{t+1}) , \w_t - \w_* \rangle
\nonumber\\
\leq &
\frac{ \epsilon_0 }{ \mu \eta T }
+ \frac{ 2\eta }{ T } \sum_{t=0}^{T-1} \| \nabla g(\w_t)^\top \nabla f(g(\w_t)) - \nabla g(\w_t)^\top \nabla f(\yh_{t+1})  \|^2
\nonumber\\
&
+ \frac{ 2\eta }{ T } \sum_{t=0}^{T-1} \| \nabla g(\w_t)^\top \nabla f(\yh_{t+1}) - \zh_{t+1}^\top \nabla f(\yh_{t+1})  \|^2
\nonumber\\
&
+ \frac{1}{T} \sum_{t=0}^{T-1} \langle \nabla g(\w_t)^\top \nabla f(g(\w_t)) - \nabla g(\w_t)^\top \nabla f(\yh_{t+1}) , \w_t - \w_* \rangle
\nonumber\\
&
+ \frac{1}{T} \sum_{t=0}^{T-1} \langle \nabla g(\w_t)^\top \nabla f(\yh_{t+1}) - \zh_{t+1}^\top \nabla f(\yh_{t+1}) , \w_t - \w_* \rangle
\nonumber\\
\leq &
\frac{ \epsilon_0 }{ \mu \eta T}
+ \frac{ 2 \eta }{ T } \cdot 8 C_g^2 L_f^2 ( 3 \log(1/\delta) + 2 ) \max(\sigma_0^2 T / m, C_g^2 D^2)
\nonumber\\
&
+ \frac{ 2 \eta }{ T } \cdot 8 C_f^2 ( 3\log(1/\delta) + 2 ) \max( \sigma_1^2 T / m, L_g^2 D^2 )
\nonumber\\
&
+ \frac{ 1 }{ T } \cdot 8 C_g L_f D ( \log(1/\delta) + 1 ) \max( \sigma_0 \sqrt{T / m}, C_g D )
\nonumber\\
&
+ \frac{ 1 }{ T } \cdot 8 C_f D ( \log(1/\delta) + 1 ) \max( \sigma_1 \sqrt{T / m}, L_g D )  ,
\end{align}
where the second inequality is due to $\mu$-optimal strong convexity and $F(\w_0) - F(\w_*) \leq \epsilon_0$,
the third inequality holds with probability at least $1 - 4\delta$ due to the four results in Lemma \ref{lemma:truncation_bounds}.

Recall $D^2 = \frac{2\epsilon_0}{\mu}$. 
To guarantee $F(\wh_T) - F(\w_*) \leq \epsilon_0 / 2$, we need to set the values for $\eta$ and $T$ as follows
\begin{align}\label{eq:RROSC_condition_of_convergence}
1) &~
\frac{ \epsilon_0 }{ \mu \eta T } 
= 
\frac{\epsilon_0}{10}
\Leftarrow
T = \frac{ 10 }{ \mu \eta },
\nonumber\\
2) &~
\frac{ 2 \eta }{ T } \cdot 8 C_g^2 L_f^2 ( 3 \log(1/\delta) + 2 ) \max(\sigma_0^2 T / m, C_g^2 D^2) \leq \frac{ \epsilon_0 }{ 10 }
\Leftarrow
\left\{
\begin{array}{l}
  \eta \leq \frac{ m \epsilon_0 }{ 160 \sigma_0^2 C_g^2 L_f^2 ( 3 \log(1/\delta) + 2 ) }   \\
  \eta \leq \frac{ 1 }{ C_g \sqrt{ 32 C_g^2 L_f^2 ( 3 \log(1/\delta) + 2 ) } }
\end{array}
\right.   ,
\nonumber\\
3) &~
\frac{ 2 \eta }{ T } \cdot 8 C_f^2 ( 3\log(1/\delta) + 2 ) \max( \sigma_1^2 T / m, L_g^2 D^2 ) \leq \frac{\epsilon_0}{10}
\Leftarrow
\left\{
\begin{array}{l}
  \eta \leq \frac{ m \epsilon_0 }{ 160 \sigma_1^2 C_f^2 ( 3 \log(1/\delta) + 2 ) }   \\
  \eta \leq \frac{ 1 }{ L_g \sqrt{ 32 C_f^2 ( 3 \log(1/\delta) + 2 ) } }
\end{array}
\right.   ,
\nonumber\\
4) &~
\frac{ 1 }{ T } \cdot 8 C_g L_f D ( \log(1/\delta) + 1 ) \max( \sigma_0 \sqrt{T / m}, C_g D ) \leq \frac{\epsilon_0}{10}
\Leftarrow
\left\{
\begin{array}{l}
  \eta \leq \frac{ m \epsilon_0 }{ 1280 \sigma_0^2 C_g^2 L_f^2 ( \log(1/\delta) + 1 )^2 }   \\
  \eta \leq \frac{ 1 }{ 16 C_g L_f ( \log(1/\delta) + 1 ) }
\end{array}
\right.
\nonumber\\
5) &~
\frac{ 1 }{ T } \cdot 8 C_f D ( \log(1/\delta) + 1 ) \max( \sigma_1 \sqrt{T / m}, L_g D )  \leq \frac{\epsilon_0}{10}
\Leftarrow
\left\{
\begin{array}{l}
  \eta \leq \frac{ m \epsilon_0 }{ 1280 \sigma_1^2 C_f^2 ( \log(1/\delta) + 1 )^2 }   \\
  \eta \leq \frac{ 1 }{ 16 L_g C_f ( \log(1/\delta) + 1 ) }   .
\end{array}
\right.
\end{align}
We can set $m = O(1)$ as a constant, instead of an increasing number in Theorem \ref{theorem:restart_MSCG}.
Therefore, we can set $\eta = O(\epsilon_0)$.

Next, we consider the $k$-th and ($k+1$)-th stage of Algorithm \ref{algorithm:RROSC}.
Suppose $\| \tilde \y_0 - g(\w_0) \| \leq \nu \sigma_0$ and $\| \tilde \z_0 - \nabla g(\w_0) \| \leq \nu \sigma_1$, which can be guaranteed by robust estimator with high probability $1-2\delta$.
Let $F(\w_k) - F(\w_*) \leq \epsilon_{k-1}$.
Then as shown above, we have probability at least $1-4\delta$ such that $F(\w_{k+1}) - F(\w_*) \leq \epsilon_{k-1} / 2$ as long as we properly set $\eta_k, T_k$ and $m_k$ for Algorithm \ref{algorithm:ROSC} according to (\ref{eq:RROSC_condition_of_convergence}).

To guarantee $F(\w_K) - F(\w_*) \leq \epsilon$ with probability $1 - 6 K \delta$, we need to set $K = \lceil \log(\epsilon_0 / \epsilon) \rceil$, which implies the total sample complexity is 
\begin{align*}
m_{tot}
= \sum_{k=1}^K T_k
= \sum_{k=1}^K O( \frac{ \log(1/\delta) }{ \mu \eta_k } )
= \sum_{k=1}^K O( \frac{ \log(1/\delta) }{ \mu \epsilon_k } )
= O(\frac{ \log(1/\delta) }{\mu \epsilon})  ,
\end{align*}
where $\eta_k = O(1/(\epsilon_{k-1}))$ according to (\ref{eq:RROSC_condition_of_convergence}) and $\epsilon_k = \epsilon_{k-1} / 2$ by assumption.

Finally, consider the truncation parameter $\lambda_k$ as follows
\begin{align*}
\lambda_k 
= &
O\big( \max( \sqrt{ T_k }, D_k ) \big)
= O\big( \max( \sqrt{ 1 / ( \mu \eta_k ) }, \sqrt{ \epsilon_{k-1} / \mu } ) \big)
\\
= &
O\big( \max( \sqrt{ 1 / ( \mu \epsilon_{k-1} ) }, \sqrt{ \epsilon_{k-1} / \mu } ) \big) .
\end{align*}
\end{proof}

\subsection{Proof of Lemma \ref{lemma:truncation_bounds} }

Before we go through our proof, we first present the following two important lemmas, in which we show six bounds related to $\| \zh_{t+1} - \nabla g(\w_t) \|$ and $\| \yh_{t+1} - g(\w_t) \|$, respectively.

\begin{lem}\label{lemma:truncation_six_bounds_grad_g}
Suppose Assumption \ref{assumption:composite_problem} holds and denote $\hat \sigma_m^2 = \sigma_1^2 / m$.
We have
\begin{align*}
1) & 
\| \zh_{t+1} - \nabla g(\w_t) \| \leq 2 L_g \| \w_t - \w_0 \| + 2 \nu \hat \sigma_m + \lambda
\\
2) &
\| \E[ \zh_{t+1} - \nabla g(\w_t) ] \| 
\leq 
\frac{\hat\sigma_m^2}{\lambda^2} ( L_g \| \w_t - \w_0 \| + \nu \hat \sigma_m ) + \frac{\hat \sigma_m^2}{\lambda}
\\
3) &
\Big( \E[ \| \zh_{t+1} - \nabla g(\w_t) \|^2 ] \Big)^{1/2} 
\leq 
\hat\sigma_m + ( L_g \| \w_t - \w_0 \| + \nu \hat\sigma_m ) \frac{\hat\sigma_m}{\lambda}
\\
4) &
\| \zh_{t+1} - \nabla g(\w_t) \|^2
\leq 
( 2L_g \| \w_t - \w_0 \| + 2\nu \hat\sigma_m + \lambda )^2
\\
5) &
\E[ \| \zh_{t+1} - \nabla g(\w_t) \|^2 ]
\leq 
( \hat\sigma_m + ( L_g \| \w_t - \w_0 \| + \nu \hat\sigma_m ) \frac{\hat\sigma_m}{\lambda} )^2
\\
6) &
\Big( \E [ \| \zh_{t+1} - \nabla g(\w_t) \|^4 ] \Big)^{1/2}
\leq 
( 2L_g \| \w_t - \w_0 \| + 2\nu \hat\sigma_m + \lambda ) \cdot ( \hat\sigma_m + ( L_g \| \w_t - \w_0 \| + \nu \hat\sigma_m ) \frac{\hat\sigma_m}{\lambda} )  . 
\end{align*}
\end{lem}

Similarly, we can have the following lemma w.r.t. $\| \yh_{t+1} - g(\w_t) \|$.
\begin{lem}\label{lemma:truncation_six_bounds_func_g}
Suppose Assumption \ref{assumption:composite_problem} holds and denote $\tilde\sigma_m^2 = \sigma_0^2 / m$.
We have
\begin{align*}
1) & 
\| \yh_{t+1} - g(\w_t) \| \leq 2 C_g \| \w_t - \w_0 \| + 2 \nu \tilde\sigma_m + \lambda
\\
2) &
\| \E[ \yh_{t+1} - g(\w_t) ] \| 
\leq 
\frac{\tilde\sigma_m^2}{\lambda^2} ( C_g \| \w_t - \w_0 \| + \nu \tilde\sigma_m ) + \frac{\tilde\sigma_m^2}{\lambda}
\\
3) &
\Big( \E[ \| \yh_{t+1} - g(\w_t) \|^2 ] \Big)^{1/2} 
\leq 
\tilde\sigma_m + ( C_g \| \w_t - \w_0 \| + \nu \tilde\sigma_m ) \frac{\tilde\sigma_m}{\lambda}
\\
4) &
\| \yh_{t+1} - g(\w_t) \|^2
\leq 
( 2C_g \| \w_t - \w_0 \| + 2\nu \tilde\sigma_m + \lambda )^2
\\
5) &
\E[ \| \yh_{t+1} - g(\w_t) \|^2 ]
\leq 
( \tilde\sigma_m + ( C_g \| \w_t - \w_0 \| + \nu \tilde\sigma_m ) \frac{\tilde\sigma_m}{\lambda} )^2
\\
6) &
\Big( \E [ \| \yh_{t+1} - g(\w_t) \|^4 ] \Big)^{1/2}
\leq 
( 2C_g \| \w_t - \w_0 \| + 2\nu \tilde\sigma_m + \lambda ) \cdot ( \tilde\sigma_m + ( C_g \| \w_t - \w_0 \| + \nu \tilde\sigma_m ) \frac{\tilde\sigma_m}{\lambda} )  . 
\end{align*}
\end{lem}

The following lemma gives Bernstein inequality for Martingale difference sequence.
\begin{lem}
\label{lemma:bernstein_inequality_0}
(Lemma 1 in \cite{peel2013empirical}, Bernstein inequality for martingales)
Suppose $X_1, ..., X_n$ is a sequence of random variables such that $0 \leq X_i \leq 1$.
Define the martingale difference sequence $\{ Y_n = \E[ X_n | X_1, ..., X_{n-1} ] - X_n \}$ and note $V_n$ the sum of the conditional variances
$$
V_n = \sum_{i=1}^n \V [X_i | X_1, ..., X_{i-1}] 
=
\sum_{i=1}^n \Big( X_i - \E[X_i | X_1, ..., X_{i-1}] \Big)^2  .
$$
Let $S_n = \sum_{i=1}^n X_i$.
Then for all $a, b \geq 0$,
$$
\P\Big[ \sum_{i=1}^n \E[ X_i | X_1, ..., X_{i-1} ] - S_n \geq a, V_n \leq b \Big]
\leq 
\exp\Big( - \frac{a^2}{2b + 2a/3} \Big)  .
$$
\end{lem}

It is can be verified that Bernstein inequality also holds for random variables $L_1, ..., L_n$ such that $0 \leq L_i \leq k$.
$$
\P\Big[ \sum_{i=1}^n L_i - \sum_{i=1}^n \E[ L_i | L_1, ..., L_{i-1} ] \geq a, \sum_{i=1}^n \V[L_i|L_1, ..., L_{i-1}] \leq b \Big]
\leq 
\exp\Big( - \frac{a^2}{2b + 2ak/3} \Big)  ,
$$
which is the form used in our proofs.
We can derive it easily by letting $L_i = -k (X_i-1)$ in Lemma \ref{lemma:bernstein_inequality_0}, as suggested in \cite{freedman1975tail}.

\begin{proof} (of Lemma \ref{lemma:truncation_bounds})

First of all, we present how we can apply Bernstein inequality to get the confidence bound for the martingale difference sequence.
Let $\gamma_0, ..., \gamma_{T-1}$ be a sequence of random variables and we can define the following upper bounds $k$, $r$ and $s$ w.r.t. $\gamma_t$:
\begin{align*}
| \gamma_t | \leq k, ~~~
| \E_t[\gamma_t] | \leq r, ~~~
\Big( \E_t[ \gamma_t^2 ] \Big)^{1/2} \leq s   ,
\end{align*}
where $\E_t$ is the expectation conditioned on $\gamma_{t-1}$.
By Bernstein inequality in Lemma \ref{lemma:bernstein_inequality_0}, for any $\tau > 0$, we can find $a > 0$ such that
\begin{align}\label{eq:apply_bernstein_inequality}
&
\P \Big[ \sum_{t=0}^{T-1} \gamma_t \geq T r + a(2\tau + 1) \Big]
\leq 
\P \Big[ \sum_{t=0}^{T-1} \gamma_t \geq \sum_{t=0}^{T-1} \E_t[\gamma_t] + a(2\tau + 1) \Big]
\nonumber\\
\leq &
\exp \Big( - \frac{ (2\tau + 1)^2 a^2 }{ 2Ts^2 + 2/3 k a (2\tau + 1) } \Big)
\leq 
\exp \Big( - \frac{ (2\tau + 1)^2 }{ 2 T s^2 / a^2 + 2/3 \cdot k(2\tau + 1) / a } \Big)  ,
\end{align}
where the first inequality is due to $T r \geq \sum_{t=0}^{T-1} \E_t[\gamma_t]$.
When the following two inequalities hold
\begin{align}\label{eq:condition_in_bernstein_inequality}
s \sqrt{T} \leq a, ~~~~~
k \leq 3a   ,
\end{align}
we have
\begin{align}\label{eq:bernstein_inequality_hold}
\P  \Big[ \sum_{t=0}^{T-1} \gamma_t \geq T r + a(2\tau + 1) \Big] 
\leq \exp\Big( - \frac{ ( 2\tau + 1 )^2 }{ 2 + 2 (2\tau + 1) } \Big)
\leq \exp(-\tau)   .
\end{align}

To this end, we need to find $a$ such that $s \sqrt{T} \leq a$ and $k \leq 3a$.
Recall that $\nu \leq \sqrt{T}$ and $\lambda = \max( L_g D, \sigma_m \sqrt{T} ) + \nu \sigma_m$.
We show how to prove this lemma in the following.

1) Let $\gamma_t = \langle \nabla g(\w_t)^\top \nabla f(\yh_{t+1}) - \zh_{t+1}^\top \nabla f(\yh_{t+1}) , \w_t - \w \rangle$ be a random variable conditioned on $\w_t$ and $\E_t[\cdot]$ denote the expectation conditioned on $\w_t$.
Recall that $f$ is $C_f$ Lipschitz continuous, and $\| \w_t - \w \| \leq \| \w_t - \w_0 \| + \| \w_0 - \w \| \leq 2D$ by Line \ref{algorithm:ROSC:line:update_w} of Algorithm \ref{algorithm:ROSC} and the assumption.
By the bounds $1), 2)$ and $3)$ in Lemma \ref{lemma:truncation_six_bounds_grad_g}, we can define the following notations of $k$, $r$ and $s$ for $\gamma_t$.
\begin{align*}
&
| \gamma_t |  
\leq 
2C_f D \| \nabla g(\w_t) - \zh_{t+1} \|
\leq 
2C_f D ( 2 L_g \| \w_t - \w_0 \| + 2\nu \hat\sigma_m + \lambda ) 
\stackrel{\Delta}{=} k ,
\\
&
| \E[ \gamma_t ] |
=  
\langle \E[ ( \nabla g(\w_t) - \zh_{t+1} )^\top \nabla f(\yh_{t+1}) ] , \w_t - \w \rangle
\\
& \qquad \quad
\leq 
2 C_f D \| \E[ \nabla g(\w_t) - \zh_{t+1} ] \|
\\
& \qquad \quad
\leq 
2 C_f D ( \frac{\hat\sigma_m^2}{\lambda^2} (L_g \| \w_t - \w_0 \| + \nu \hat\sigma_m) + \frac{\hat\sigma_m^2}{\lambda} )
\stackrel{\Delta}{=} r,
\\
&
\Big( \E[ \gamma_t^2 ] \Big)^{1/2}
\leq 
\Big( \E[ (2 C_f D)^2 \| \nabla g(\w_t) - \zh_{t+1} \|^2 ] \Big)^{1/2}
\\
& \qquad \qquad \quad
= 
2 C_f D \Big( \E[ \| \nabla g(\w_t) - \zh_{t+1} \|^2 ] \Big)^{1/2}
\\
& \qquad \qquad \quad
\leq 
2 C_f D ( \hat\sigma_m + ( L_g \| \w_t - \w_0 \| + \nu \hat\sigma_m ) \frac{\hat\sigma_m}{\lambda} )  
\stackrel{\Delta}{=} s.
\end{align*}

Then to make $s\sqrt{T} \leq a$ and $k/3 \leq a$ hold, we let $a = 4C_f D \max( \hat\sigma_m \sqrt{T}, L_g D )$:
\begin{align*}
s \sqrt{T} 
= &
2C_f D (\hat\sigma_m + ( L_g \| \w_t - \w_0 \| + \nu \hat\sigma_m ) \frac{\hat\sigma_m}{\lambda} ) \sqrt{T}
\\
\leq &
2 C_f D \cdot 2 \hat\sigma_m \sqrt{T}
\leq 
4 C_f D \max( \hat\sigma_m \sqrt{T}, L_g D )
= a  ,
\\
k / 3
= &
2/3 C_f D ( 2L_g D + 2 \nu \hat\sigma_m + \lambda )
\\
\leq &
2/3 C_f D \cdot 3 \lambda 
=
2 C_f D \lambda
\\
\leq &
2 C_f D ( \max( \hat\sigma_m \sqrt{T}, L_g D ) + \hat\sigma_m \sqrt{T} )
\\
\leq &
4 C_f D \max( \hat\sigma_m \sqrt{T}, L_g D ) )
= a  ,
\end{align*}
which satisfy (\ref{eq:condition_in_bernstein_inequality}).

We also have the upper bound for $r$ as follows
\begin{align*}
r = &
2 C_f D ( \frac{\hat\sigma_m^2}{\lambda^2} (L_g D + \nu \hat\sigma_m) + \frac{\hat\sigma_m^2}{\lambda} )
\\
\leq &
4 C_f D \frac{\hat\sigma_m^2}{\lambda}
=
4 C_f D \frac{\hat\sigma_m^2}{ \max( L_g D, \hat\sigma_m \sqrt{T} ) + \nu \hat\sigma_m }
\\
\leq &
4 C_f D \frac{\hat\sigma_m^2}{ \sigma_m \sqrt{T} }
= 
4 C_f D \frac{\hat\sigma_m}{ \sqrt{T} }  ,
\end{align*}
which, with the setting of $a$, implies the LHS of (\ref{eq:bernstein_inequality_hold}):
\begin{align*}
& \P \Big[ \sum_{t=0}^{T-1} \gamma_t \geq T r + 4C_f D \max(\hat\sigma_m \sqrt{T}, L_g D) ( 2\tau + 1 ) \Big]
\\
\geq &
\P \Big[ \sum_{t=0}^{T-1} \gamma_t \geq T \cdot 4 C_f D \frac{\hat\sigma_m}{ \sqrt{T} } + 4C_f D \max(\hat\sigma_m \sqrt{T}, L_g D) ( 2\tau + 1 ) \Big] 
\\
= &
\P \Big[ \sum_{t=0}^{T-1} \gamma_t \geq 4 C_f D \hat\sigma_m \sqrt{T} + 4C_f D \max(\hat\sigma_m \sqrt{T}, L_g D) ( 2\tau + 1 ) \Big]
\\
\geq &
\P \Big[ \sum_{t=0}^{T-1} \gamma_t \geq 8 C_f D \max(\hat\sigma_m \sqrt{T}, L_g D) ( \tau + 1 ) \Big] .
\end{align*}

Plugging the above lower bound to (\ref{eq:bernstein_inequality_hold}) leads to 
\begin{align*}
&
\P \Big[ \sum_{t=0}^{T-1} \gamma_t \geq 8 C_f D \max(\hat\sigma_m \sqrt{T}, L_g D) ( \tau + 1 ) \Big]
\leq 
\exp( - \tau )  .
\end{align*}
Finally, replacing $\tau = \log(1/\delta)$, we prove the result.

2) This proof is as the one of result 1), where the difference is that we use $C_g$-Lipschitz continuity of $g$ in truncation, instead of $L_g$-Lipschitz continuity of $\nabla g$.
Since the proof is very similar, we skip the proof and state that the following inequality holds with probability at least $1-\delta$:
\begin{align*}
&
\sum_{t=0}^{T-1} \langle \nabla g(\w_t)^\top \nabla f(g(\w_t)) - \nabla g(\w_t)^\top \nabla f(\yh_{t+1}) , \w_t - \w \rangle 
\leq 
8 C_g L_f D \max( \tilde\sigma_m \sqrt{T}, C_g D ).
\end{align*}

3) Let $\gamma_t = \| \nabla g(\w_t)^\top \nabla f(\yh_{t+1}) - \zh_{t+1}^\top \nabla f(\yh_{t+1}) \|^2$ be a random variable conditioned on $\w_t$.
By the bounds $4), 5)$ and $6)$ in Lemma \ref{lemma:truncation_six_bounds_grad_g}, we can define the following notations of $k, r$ and $s$ for $\gamma_t$: 
\begin{align*}
&
| \gamma_t |
\leq C_f^2 \| \nabla g(\w_t) - \zh_{t+1} \|^2
\leq C_f^2 ( 2L_g \| \w_t - \w_0 \| + 2 \nu \hat\sigma_m + \lambda )^2
\stackrel{\Delta}{=} k ,
\\
&
| \E[\gamma_t] |
= \E[ \| \nabla g(\w_t)^\top \nabla f(\yh_{t+1}) - \zh_{t+1}^\top \nabla f(\yh_{t+1}) \|^2 ]
\leq C_f^2 \E[ \| \nabla g(\w_t) - \zh_{t+1} \|^2 ]
\\
& \qquad \quad
\leq C_f^2 ( \hat\sigma_m + ( L_g \| \w_t - \w_0 \| + \nu \hat\sigma_m ) \frac{\hat\sigma_m}{\lambda} )^2
\stackrel{\Delta}{=} r  ,
\\
&
\Big( \E[ \gamma_t^2 ] \Big)^{1/2}
\leq C_f^2 ( 2L_g \| \w_t - \w_0 \| + 2 \nu \hat\sigma_m + \lambda ) \cdot ( \hat\sigma_m + ( L_g \| \w_t - \w_0 \| + \nu \hat\sigma_m ) \frac{\hat\sigma_m}{\lambda} )
\stackrel{\Delta}{=} s  .
\end{align*}

Then to make $s\sqrt{T} \leq a$ and $k/3 \leq a$ hold, we let $a = 12 C_f^2 \max( \hat\sigma_m^2 T, L_g^2 D^2 )$:
\begin{align*}
s\sqrt{T}
= & 
C_f^2 ( 2L_g \| \w_t - \w_0 \| + 2 \nu \hat\sigma_m + \lambda ) \cdot ( \hat\sigma_m + ( L_g \| \w_t - \w_0 \| + \nu \hat\sigma_m ) \frac{\hat\sigma_m}{\lambda} ) \sqrt{T}
\\
\leq &
C_f^2 \cdot 3 \lambda \cdot 2 \hat\sigma_m \sqrt{T} 
= 6 C_f^2 \hat\sigma_m \sqrt{T} \cdot ( \max( \hat\sigma_m \sqrt{T}, L_g D ) + \nu \hat\sigma_m )
\\
\leq &
12 C_f^2 \max( \hat\sigma_m^2 T, L_g^2 D^2 )
= a ,
\\
k/3
= & 
C_f^2 ( 3 \lambda )^2 / 3
= 3 C_f^2 ( \max( \hat\sigma_m \sqrt{T}, L_g D ) + \nu \hat\sigma_m )^2
\\
\leq &
12 C_f^2 \max( \hat\sigma_m^2 T, L_g^2 D^2 )
= a  .
\end{align*}

We also have the upper bound for $r$ as follows
\begin{align*}
r 
= C_f^2 ( \hat\sigma_m + ( L_g D + \hat\sigma_m \sqrt{T} ) \frac{\hat\sigma_m}{\lambda} )^2
\leq 4 C_f^2 \hat\sigma_m^2   ,
\end{align*}
where we use $\lambda = \max( \hat\sigma_m \sqrt{T}, D ) + \nu \hat\sigma_m \geq \max( \hat\sigma_m \sqrt{T}, L_g D )$ .

Using the above upper bound of $r$ and the setting of $a$, by (\ref{eq:bernstein_inequality_hold}), we have
\begin{align*}
&
\P \Big[ \sum_{t=0}^{T-1} \gamma_t \geq 8 C_f^2 ( 2 + 3\tau ) \max( \hat\sigma_m^2 T, L_g^2 D^2 ) \Big]
\\
\leq &
\P \Big[ \sum_{t=0}^{T-1} \gamma_t \geq 4 C_f^2 \hat\sigma_m^2 T + 12 C_f^2 \max( \hat\sigma_m^2 T, L_g^2 D^2 ) ( 2\tau + 1 ) \Big]
\\
\leq & 
\P \Big[ \sum_{t=0}^{T-1} \gamma_t \geq T r + 12 C_f^2 \max( \hat\sigma_m^2 T, L_g^2 D^2 ) ( 2\tau + 1 ) \Big]
\\
\leq &
\exp(-\tau) .
\end{align*}
Finally, replacing $\tau = \log(1/\delta)$, we prove the result.

4) This proof is as the one of result $3)$, where the difference is that we use $C_g$-Lipschitz continuity of $g$ in truncation, instead of $L_g$-Lipschitz continuity of $\nabla g$.
Since the proof is very similar, we skip the proof and state that the following inequality holds with probability at least $1-\delta$:
\begin{align*}
&
\sum_{t=0}^{T-1} \| \nabla g(\w_t)^\top \nabla f(g(\w_t)) - \nabla g(\w_t)^\top \nabla f(\yh_{t+1}) \|^2 
\leq 
8 C_g^2 L_f^2 ( 3\log(1/\delta) + 2 ) \max( \tilde\sigma_m^2 T, C_g^2 D^2 ).
\end{align*}



\end{proof}

\subsection{Proof of Lemma \ref{lemma:truncation_six_bounds_grad_g} }\label{section:proof:lemma:truncation_six_bounds_grad_g}
\begin{proof} (of Lemma \ref{lemma:truncation_six_bounds_grad_g})
First we define two indicators
\begin{align*}
\chi_t = {\bf 1}_{ \| \z_{t+1} - \tilde \z_0 \| \geq L_g \| \w_t - \w_0 \| + \nu \hat \sigma_m + \lambda } ,
~~~
\omega_t = {\bf 1}_{ \| \z_{t+1} - \nabla g(\w_t) \| \geq \lambda }   .
\end{align*}

It can be verified that $\chi_t \leq \omega_t$ since $\omega_t = 1$ is a necessary condition of $\chi_t = 1$, but not sufficient:
\begin{align*}
\| \z_{t+1} - \tilde \z_0 \|
\leq &
\| \z_{t+1} - \nabla g(\w_t) \| + \| \nabla g(\w_t) - \tilde \z_0 \|
\\
\leq &
\| \z_{t+1} - \nabla g(\w_t) \| + \| \nabla g(\w_t) - \nabla g(\w_0) \| + \| \nabla g(\w_0) - \tilde \z_0 \|
\\
\leq &
\| \z_{t+1} - \nabla g(\w_t) \| + L_g \| \w_t - \w_0 \| + \nu \hat \sigma_m   ,
\end{align*}
where the last inequality is due to $L_g$-smoothness of $g$.
One one hand, when $\| \z_{t+1} - \nabla g(\w_t) \| \leq \lambda$, i.e., $\omega_t = 0$, we must have $\z_{t+1} - \tilde \z_0 \leq \lambda + L_g \| \w_t - \w_0 \| + \nu \hat \sigma_m$, i.e., $\chi_t = 0$.
On the other hand, when $\| \z_{t+1} - \nabla g(\w_t) \| \geq \lambda$, i.e., $\omega_t = 1$, it is unknown whether $\z_{t+1} - \tilde \z_0 \geq \lambda + L_g \| \w_t - \w_0 \| + \nu \sigma_m$, i.e., $\chi_t = 1$.
Therefore, $\omega_t \geq \chi_t$.

The above inequalities also show
\begin{align*}
\| \nabla g(\w_t) - \tilde \z_0 \| 
\leq 
L_g \| \w_t - \w_0 \| + \nu \sigma_m  .
\end{align*}

To better show the relation between $\zh_{t+1}$ and $\nabla g(\w_t)$, we can also rewrite the following three equivalent forms:
\begin{align}
\label{eq:RMSCG_ref_three_bounds_three_forms_1}
\zh_{t+1} - \nabla g(\w_t)
= &
\chi_t ( \tilde \z_0 - \nabla g(\w_t) )
+ (1-\chi_t) ( \z_{t+1} - \nabla g(\w_t) )
\\
\label{eq:RMSCG_ref_three_bounds_three_forms_2}
= &
\chi_t ( \tilde \z_0 - \z_{t+1} ) + (\z_{t+1} - \nabla g(\w_t))
\\
\label{eq:RMSCG_ref_three_bounds_three_forms_3}
= &
(1-\chi_t) ( \z_{t+1} - \tilde \z_0 ) + (\tilde \z_0 - \nabla g(\w_t))   .
\end{align}

$1$) To prove the first inequality, i.e., the absolute bound of $\| \zh_{t+1} - \nabla g(\w_t) \|$, we use (\ref{eq:RMSCG_ref_three_bounds_three_forms_3}) and have
\begin{align*}
\| \zh_{t+1} - \nabla g(\w_t) \|
= &
\| (1-\chi_t) ( \z_{t+1} - \tilde \z_0 ) + (\tilde \z_0 - \nabla g(\w_t)) \|
\\
\leq &
(1-\chi_t) \| \z_{t+1} - \tilde \z_0 \| + \| \tilde \z_0 - \nabla g(\w_t) \|
\\
\leq &
L_g \| \w_t - \w_0 \| + \nu \sigma_m + \lambda
+ L_g \| \w_t - \w_0 \| + \nu \sigma_m
\\
= &
2( L_g \| \w_t - \w_0 \| + \nu \sigma_m ) + \lambda   .
\end{align*}

$2$) To prove the second result w.r.t. $\| \E[ \zh_{t+1} - \nabla g(\w_t) ] \|$, we first take a closer look to $\E[ \zh_{t+1} - \nabla g(\w_t) ]$ from the perspective of (\ref{eq:RMSCG_ref_three_bounds_three_forms_2}) as follows
\begin{align*}
\E[ \zh_{t+1} - \nabla g(\w_t) ]
= &
\E[ \chi_t (\tilde \z_0 - \z_{t+1}) + ( \z_{t+1} - \nabla g(\w_t) ) ]
\\
= &
\E[ \chi_t (\tilde \z_0 - \z_{t+1}) ]
\\
= & 
\E[ \chi_t ( \tilde \z_0 - \nabla g(\w_t) + \nabla g(\w_t) - \z_{t+1} ) ] ,
\end{align*}
where the second equality is due to $\E[ \z_{t+1} - \nabla g(\w_t) ] = 0$.

Now we consider $\| \E[ \zh_{t+1} - \nabla g(\w_t) ] \|$ as follows
\begin{align*}
\| \E[ \zh_{t+1} - \nabla g(\w_t) ] \|
= &
\| \E[ \chi_t ( \tilde \z_0 - \nabla g(\w_t) + \nabla g(\w_t) - \z_{t+1} ) ] \|
\\
\leq &
\E \| \chi_t ( \tilde \z_0 - \nabla g(\w_t) + \nabla g(\w_t) - \z_{t+1} ) \|
\\
\leq &
\E \| \chi_t ( \tilde \z_0 - \nabla g(\w_t) ) \| + \E \| \chi_t ( \nabla g(\w_t) - \z_{t+1} ) \|
\\
= &
\E [ \chi_t \| ( \tilde \z_0 - \nabla g(\w_t) ) \| ] + \E [ \chi_t \| \nabla g(\w_t) - \z_{t+1} \| ]
\\
\leq &
\E [ \omega_t ( L_g \| \w_t - \w_0 \| + \nu \sigma_m ) ] + \E [ \omega_t \lambda ]
\\
\leq &
\frac{\sigma_m^2}{\lambda^2} ( L_g \| \w_t - \w_0 \| + \nu \sigma_m ) + \frac{\sigma_m^2}{\lambda},
\end{align*}
where the first inequality is due to Jensen's inequality and the convexity of $\| \cdot \|$.
The third inequality is due to $\chi_t \leq \omega_t$.
The last inequality is due to Markov inequality for a random variable $X$ and constant $a$:
\begin{align*}
\P [ X \geq a ] \leq \frac{\E[X]}{a}  .
\end{align*}
Here we specifically have
\begin{align*}
\P [ \omega_t = 1 ]
=
\P [ \| \z_{t+1} - \nabla g(\w_t) \|^2 \geq \lambda^2 ] 
\leq 
\frac{ \E [ \| \z_{t+1} - \nabla g(\w_t) \|^2 ] }{\lambda^2}
\leq
\frac{ \sigma_m^2 }{\lambda^2}  .
\end{align*}

$3$) To prove the third result w.r.t. $\Big( \E [ \| \zh_{t+1} - \nabla g(\w_t) \|^2 ] \Big)^{1/2}$, we use (\ref{eq:RMSCG_ref_three_bounds_three_forms_3}) and have
\begin{align*}
&
\Big( \E [ \| \zh_{t+1} - \nabla g(\w_t) \|^2 ] \Big)^{1/2}
\\
= &
\Big( \E [ \chi_t \| \tilde \z_0 - \nabla g(\w_t) \|^2 ] \Big)^{1/2}
+ \Big( \E [ (1-\chi_t) \| \z_{t+1} - \nabla g(\w_t) \|^2 ] \Big)^{1/2}
\\
\leq &
\Big( \E [ \chi_t \| \tilde \z_0 - \nabla g(\w_t) \|^2 ] \Big)^{1/2}
+ \Big( \E [ \| \z_{t+1} - \nabla g(\w_t) \|^2 ] \Big)^{1/2}
\\
\leq &
\sigma_m + \Big( \E [ \chi_t ( L_g \| \w_t - \w_0 \| + \nu \sigma_m )^2 ] \Big)^{1/2}
\\
\leq &
\sigma_m + ( L_g \| \w_t - \w_0 \| + \nu \sigma_m ) ( \E [ \omega_t ] )^{1/2}
\\
\leq &
\sigma_m + ( L_g \| \w_t - \w_0 \| + \nu \sigma_m ) \frac{\sigma_m}{\lambda}  ,
\end{align*}
where the first inequality is due to $(1-\chi_t) \leq 1$,
the second inequality is due to $\E [ \| \z_{t+1} - \nabla g(\w_t) \|^2 ] \leq \sigma_m^2$ and $\| \tilde \z_0 - \nabla g(\w_t) \| \leq L_g \| \w_t - \w_0 \|$,  
the third inequality is due to $\chi_t \leq 1$, and
the last inequality is due to $\E[\omega_t] \leq \frac{\sigma_m^2}{\lambda^2}$.

$4$) For $\| \zh_{t+1} - \nabla g(\w_t) \|^2$, we can simply apply the result $1$) for $\| \zh_{t+1} - \nabla g(\w_t) \|$:
\begin{align*}
\| \zh_{t+1} - \nabla g(\w_t) \|^2
\leq 
( 2 L_g \| \w_t - \w_0 \| + 2 \nu \sigma_m + \lambda )^2  .
\end{align*}

$5$) By applying the third result we have proved, we have
\begin{align*}
\E [ \| \zh_{t+1} - \nabla g(\w_t) \|^2 ]
\leq 
\Big( \E [ \| \zh_{t+1} - \nabla g(\w_t) \|^2 ] \Big)^{2/2 }
=
( \sigma_m + ( L_g \| \w_t - \w_0 \| + \nu \sigma_m ) \frac{\sigma_m}{\lambda} )^2  .
\end{align*}

$6$) For the last result, we have
\begin{align*}
&
\Big( \E[ \| \zh_{t+1} - \nabla g(\w_t) \|^4 ] \Big)^{1/2}
\\
= &
\Big( \E[ \| \zh_{t+1} - \nabla g(\w_t) \|^2 \cdot \E[ \| \zh_{t+1} - \nabla g(\w_t) \|^2 ] \Big)^{1/2}
\\
\leq &
\Big( \E[ \| \zh_{t+1} - \nabla g(\w_t) \|^2 \cdot ( 2L_g \| \w_t - \w_0 \| + 2 \nu \sigma_m + \lambda ) \Big)^{1/2}
\\
\leq &
( 2L_g \| \w_t - \w_0 \| + 2 \nu \sigma_m + \lambda ) \cdot ( \sigma_m + ( L_g \| \w_t - \w_0 \| + \nu \sigma_m ) \frac{\sigma_m}{\lambda} )  ,
\end{align*}
where the first and second inequalities are due to results $4$) and $5$), respectively.

We omit the proof of Lemma \ref{lemma:truncation_six_bounds_func_g} since it is highly anagolous to the proof of Lemma \ref{lemma:truncation_six_bounds_grad_g}.
\end{proof}

\section{Proof of Lemma \ref{lemma:square_loss_is_scvx}}
\begin{proof}
Let $F(\w)$ denote the objective in (\ref{eqn:dros}). 
Denote a feature-label pair as $\xi = (\x, y)$.
Since $\E_{\xi\sim \mathbb{P}_i} [\x \x^T]$ is positive semi-definite, there exist $A_i \in \mathbb{R}^{d\times d}$  such that
\begin{align}
\begin{split}
\mathcal{L}_i(\w) &= \E_{\xi\sim \mathbb{P}_i} \ell(\w; \xi) = \E_{\xi\sim \mathbb{P}_i} [(\w^T \x - y)^2]  \\
&=\E_{\xi\sim \mathbb{P}_i} [\w^T \x \x^T \w - 2 \w^T \x y + y^2] \\
& = \w^T A_i^T A_i \w - \b_i^T \w + c_i \\
\end{split} 
\end{align} 
where $\b_i=2 \E_{\xi\sim \mathbb{P}_i}[y\x]$ and 
$c_i = \E_{\xi\sim \mathbb{P}_i}[y^2]$.
Note that the components of $\b_i$ and $c_i$ are constants that do not depend on $\w$.

And we know there exists $A \in \mathbb{R}^{d\times f}$ is matrix such that $A^T A =  \sum\limits_{i=1}^{m} p_i A_i^T A_i$ since $\sum\limits_{i=1}^{m} p_i A_i^T A_i$ is positive semi-positive. 
Define 
\begin{align} 
\begin{split} 
\phi (\w, \p) &= \sum\limits_{i=1}^{m} p_i \mathcal{L}_i (\w) - h(\p) + r(\w) \\
& = \sum\limits_{i=1}^{m} p_i \w^T A_i^T A_i \w -  \sum\limits_{i=1}^{m} \b_i^T \w +  \sum\limits_{i=1}^{m} c_i - h(\p) + r(\w) \\ 
& = \w^T A^T A \w - \b^T \w + c- h(\p) + r(\w) \\
& = \|A \w\|_2^2 - \b^T \w + c- h(\p) + r(\w) 
\end{split} 
\end{align} 
where 
$\b = \sum\limits_{i=1}^{m} \b_i$,  
$c = \sum\limits_{i=1}^{m} c_i$. 
Then we have $F(\w) = \max\limits_{\p} \phi(\w, \p)$.
For any fixed $\p$, applying Lemma 2 of \cite{gong2014linear}, we know that there is a $\mu>0$ such that
\begin{align}
\begin{split}
\phi(\w, \p) - \arg\min\limits_{\w'} \phi(\w', \p) \geq \frac{\mu}{2} \|\w-\w'\|^2 
\end{split}
\label{eqn:local_osc_1}
\end{align}

Denote $\hat{p}(\w) = \arg\max\limits_{\p \in \Delta_m} \phi(\w, \p)$.
Let $\w_* = \arg\min\limits_{\w} F(\w)$, 
and $\p_* = \arg\max\limits_{\w} \phi(\w_*, \p)$, i.e., $(\w_*, \p_*)$ is a saddle point of $\phi(\w, \p)$.
Then by (\ref{eqn:local_osc_1}), we have
\begin{align}
\frac{\mu}{2} \|\w - \w_*\|^2 \leq \phi(\w, \p_*) - \phi(\w_*, \p_*)
\leq \phi(\w, \hat{p}(\w))  - \phi(\w_*, \p_*) = F(\w) - F(\w_*) 
\end{align}
\end{proof}

\end{document}